\documentclass[letterpaper, 10 pt, conference]{ieeeconf}  

\IEEEoverridecommandlockouts                              

\title{\LARGE \bf
Optimization of Utility-based Shortfall Risk:\\ A Non-asymptotic Viewpoint
}

\usepackage{amsmath}
\usepackage{amssymb}
\usepackage{relsize}
\usepackage{comment}
\usepackage{xfrac,bigints}
\usepackage[ruled,vlined,noend]{algorithm2e}
\usepackage{xcolor}
\usepackage{times}
\allowdisplaybreaks
\usepackage[hidelinks]{hyperref}
\usepackage{cleveref}
\crefname{subsection}{section}{sections}
\crefname{lemma}{lemma}{lemmas}
\crefname{assumption}{assumption}{assumptions}
\usepackage{soul}
\usepackage{todonotes}
\usepackage{xcolor}
\usepackage{xspace}
\usepackage{adjustbox}
\usepackage{multirow}
\usepackage{graphicx}
\usepackage{textcomp} 

\newcommand{\Leb}{\mathit{L}}
\newcommand{\Rel}{\mathbb{R}}
\newcommand{\Exp}{\mathbb{E}}
\newcommand{\xl}{\mathcal{X}_l}
\newcommand{\Z}{\mathbf{Z}}
\newcommand{\Zh}{\mathbf{\hat{Z}}}

\newcommand\norm[1]{\left\lVert#1\right\rVert} 
\def\argmin{\mathop{\rm arg\,min}}
\newcommand\numberthis{\addtocounter{equation}{1}\tag{\theequation}}

\newtheorem{theorem}{Theorem}[section]
\newtheorem{lemma}{Lemma}
\newtheorem{corollary}{Corollary}[lemma]
\newtheorem{proposition}{Proposition}

\newtheorem{assumption}{Assumption}
\newtheorem{assumptionalt}{Assumption}[assumption]
\newtheorem{lemmaalt}{Lemma}[lemma]

\newtheorem{remark}{Remark}
\newtheorem{definition}{Definition}

\newenvironment{assumptionp}[1]{
  \renewcommand\theassumptionalt{#1}
  \assumptionalt
}{\endassumptionalt}

\newenvironment{lemmap}[1]{
  \renewcommand\thelemmaalt{#1}
  \lemmaalt
}{\endlemmaalt}

\author{Sumedh Gupte$^{1}$ and Prashanth L.A.$^{2}$ and Sanjay P. Bhat$^{3}$}

\begin{document}

\maketitle
\thispagestyle{empty}
\pagestyle{empty}

\begin{abstract}

We consider the problems of estimation and optimization of utility-based shortfall risk (UBSR), which is a popular risk measure in finance. In the context of UBSR estimation, we derive a non-asymptotic bound on the mean-squared error of the classical sample average approximation (SAA) of UBSR. Next, in the context of UBSR optimization, we derive an expression for the UBSR gradient under a smooth parameterization. This expression is a ratio of expectations, both of which involve the UBSR. We use SAA for the numerator as well as denominator in the UBSR gradient expression to arrive at a biased gradient estimator. We derive non-asymptotic bounds on the estimation error, which show that our gradient estimator is asymptotically unbiased. We incorporate the aforementioned gradient estimator  into a stochastic gradient (SG) algorithm for UBSR optimization. Finally, we derive non-asymptotic bounds that quantify the rate of convergence of our SG algorithm for UBSR optimization.

\end{abstract}


\section{INTRODUCTION}
\label{sec:intro}
Optimizing risk is important in several application domains, e.g., finance, transportation to name a few.  Value-at-Risk (VaR)~\cite{jorion1997value,basak-shapiro-var}, Conditional Value-at-Risk (CVaR)~\cite{Uryasev2001,rockafellar2000optimization}  are two popular risk measures. The risk measure VaR, which is a quantile of the underlying distribution, is not the preferred choice owing to the fact that it is not sub-additive. In a financial context, the latter property ensures that the risk measure is consistent with the intuition that diversification does not increase risk. CVaR as a risk measure satisfies sub-additivity property and falls in the category  of coherent risk measures~\cite{ACERBI20021487}. A class of risk measures that subsumes coherency is convex risk measures~\cite{FollmerSchied2004}. A prominent convex risk measure is utility-based shortfall risk (UBSR).

Financial applications rely heavily on efficient risk assessment techniques and employ multitude of risk measures for risk estimation. Risk optimization too involves risk estimation as a sub-procedure for finding solutions to optimal decision-making problems in finance. In this paper, we consider the problems of UBSR estimation and optimization. UBSR is a convex risk measure that has a few advantages over the popular CVaR risk measure, namely (i) UBSR is invariant under randomization, while CVaR is not, see \cite{dunkel2010stochastic}; (ii) Unlike CVaR, which only considers the values that the underlying random variable takes beyond VaR, the loss function in UBSR can be chosen to encode the risk preference for each value that  the underlying random variable takes.  Thus, in the context of both risk estimation and optimization, UBSR is a viable alternative to the industry standard risk measures, namely, VaR and CVaR.

The existing works on UBSR are restricted to the case where the underlying random variables are bounded, cf. \cite{hegde2021ubsr}. In this paper, we extend the UBSR formalism to unbounded random variables. 
We now summarize our contributions below.
\begin{enumerate}
    \item We extend UBSR to cover unbounded random variables that satisfy certain integrability requirements, and establish conditions under which UBSR is a convex risk measure.
	    \item For a sample average approximation (SAA) of UBSR, which is proposed earlier in the literature, we derive a mean-squared error bound under a weaker assumption on the underlying loss function. More precisely, we assume  the loss function to be bounded above by a quadratic function, while the corresponding bound in \cite{JMLR-LAP-SPB} assumed linear growth.
    \item For the problem of UBSR optimization with a vector parameter, we derive an expression for the gradient of UBSR. Using this expression, we propose an $m$-sample gradient estimator, and establish a $\mathcal{O}(1/m)$ mean-squared error bound.
    \item We design a stochastic gradient (SG) algorithm using the gradient estimator above, and  derive a non-asymptotic bound of $\mathcal{O}(1/n)$ under a strong convexity assumption. Here $n$ denotes the number of iterations of the SG algorithm for UBSR optimization.
\end{enumerate}

\noindent\textit{Related work.} The authors in~\cite{FollmerSchied2004} have introduced UBSR for bounded random variables and the authors in~\cite{giesecke-risk-large-losses} have illustrated several desirable properties of UBSR. 
In real-world financial markets, the financial positions continuously evolve over time, and so must their risk estimates.
 In~\cite{weber-2006-distribution-invariant}, the authors showed that UBSR can be used for dynamic evaluation of such financial positions. In~\cite{dunkel2010stochastic}, the authors have proposed estimators based on a stochastic root finding procedure and they provide only asymptotic convergence guarantees. In~\cite{zhaolin2016ubsrest}, the authors have used sample average approximation (SAA) procedure for UBSR estimation and have proposed an estimator for the UBSR derivative which can be used for risk optimization in the case of a scalar decision parameter. They establish asymptotic convergence guarantees for UBSR estimation and also show that the UBSR derivative is asymptotically unbiased. In~\cite{hegde2021ubsr}, the authors perform non-asymptotic analysis for the scalar UBSR optimization, while employing a stochastic root finding technique for UBSR estimation. In comparison to these works, we would like to the note the following aspects: (i) Unlike \cite{dunkel2010stochastic,zhaolin2016ubsrest}, we provide non-asymptotic bounds on the mean-squared error of the UBSR estimate from a procedure that is computationally efficient; (ii) We consider UBSR optimization for a vector parameter, while earlier works (cf. \cite{zhaolin2016ubsrest,hegde2021ubsr}) consider the scalar case; (iii) We analyze a SG-based algorithm in the non-asymptotic regime for UBSR optimization, while \cite{zhaolin2016ubsrest} provides an asymptotic guarantee for the UBSR derivative estimate; (iv) In \cite{hegde2021ubsr}, UBSR optimization using a gradient-based algorithm has been proposed for the case of scalar parameterization. Unlike \cite{hegde2021ubsr}, we derive a general (multivariate) expression for the UBSR gradient, leading to an estimator that is subsequently employed in a stochastic gradient algorithm mentioned above. A vector parameter makes the bias/variance analysis of UBSR gradient estimate challenging as compared to the scalar counterpart, which is analyzed in \cite{hegde2021ubsr}. 


The rest of the paper is organized as follows:
In \Cref{sec:prelim}, we describe the symbols and notations that are used in the paper. In \Cref{sec:pb} we define UBSR for unbounded random variables, construct an analytical expression for UBSR, and show that it is a convex risk measure. In \Cref{sec:ubsr-est}, we present SAA-based UBSR estimation algorithm and derive its error bounds. In \Cref{sec:ubsr-opt} we derive the UBSR gradient theorem, propose a stochastic gradient scheme for UBSR minimization and derive its convergence rate and analyze sample efficiency for the cases of constant batchsize and varying batchsize. In \Cref{sec:conclusions}, we provide the concluding remarks. Due to space constraints, we provide detailed proofs in the longer version of this submission, available in \cite{gupte2023optimization}.

\section{Preliminaries}
\label{sec:prelim}
For $p\in[1,\infty)$, the $p$-norm of a vector $\mathbf{v} \in \Rel^d$ is given by
\begin{equation*}\label{eq:vector-norm}
    \norm{\mathbf{v}}_p \triangleq \Bigg( \sum_{i=1}^d |v_i|^p \Bigg)^\frac{1}{p},
\end{equation*}
while $\norm{\mathbf{v}}_\infty$ denotes the supremum norm. Matrix norms induced by the vector $p$-norm are denoted by $\norm{\cdot}_p$, where the special cases of $p=1$ and $p=\infty$ equal the maximum absolute column sum and maximum absolute row sum, respectively. The spectral norm of a matrix is denoted by $\norm{\cdot}$.

Let $(\Omega, \mathcal{F}, \mathit{P})$ be a standard-Borel probability space. Let $\Leb^0$ denote the space of $\mathcal{F}$-measurable, real random variables and let $\Exp(\cdot)$ denote the expectation under $\mathit{P}$. For $p \in [1,\infty)$, let $\big(\Leb^p, \norm{\cdot}_{\Leb^p})\big)$ denote the normed vector space of random variables $X: \Omega \to \Rel$ in $\Leb^0$ for which 
$\norm{X}_{\Leb^p} \triangleq \left( \Exp \big[ |X|^p \big] \right)^\frac{1}{p}$ is finite. Further, we let $\big(\Leb^\infty, \norm{\cdot}_{\Leb^\infty})\big)$ denote the normed vector space of random variables $X: \Omega \to \Rel$ in $\Leb^0$, for which, $    \norm{X}_{\Leb^\infty} \triangleq \inf \{M \in \Rel : |X| \leq M \text{ a.s.} \}$ is finite. Let $p\in[1,\infty)$ and let $\mathbf{Z}$ be a random vector such that each $Z_i$ is $\mathcal{F}$-measurable and has finite $p^{th}$ moment. Then the $\mathit{L}^p$-norm of $\Z$ is defined by $\norm{\mathbf{Z}}_{L^p} \triangleq \left( \Exp \Big[ \norm{\mathbf{Z}}_p^p\Big] \right)^\frac{1}{p}$. 

Let $\mu_X$ and $\mu_Y$ denote the marginal distributions of random vectors $\mathbf{X}$ and $\mathbf{Y}$ respectively. Let $\mathcal{H}(\mu_{\mathbf{X}},\mu_{\mathbf{Y}})$ denote the set of all joint distributions having $\mu_{\mathbf{X}}$ and $\mu_{\mathbf{Y}}$ as the marginals. Then, for every $p\geq 1,\mathcal{T}_p(\mu_{\mathbf{X}},\mu_{\mathbf{Y}}) \triangleq \inf \left\{ \int \norm{x-y}^p \eta(dx,dy) : \eta \in \mathcal{H}(\mu_{\mathbf{X}},\mu_{\mathbf{Y}}) \right\}$ denotes the optimal transpost cost associated with $\mathbf{X}$ and $\mathbf{Y}$, and $\mathcal{W}_p(\mu_{\mathbf{X}},\mu_{\mathbf{Y}}) = \left(\mathcal{T}_p(\mu_{\mathbf{X}},\mu_{\mathbf{Y}})\right)^{1/p}$ denotes the Wasserstein distance. Note that vectors and random vectors are distinguished from their scalar counterparts by the use of boldface fonts.

Throughout this paper, we shall use $l:\Rel \to \Rel$ to denote a continuous function that models a loss function, for instance, of a financial position. Let $\mathcal{X}_l$ denote the space of random variables $ X \in \Leb^0$ for which the collection of random variables $\{l(-X-t) : t \in \Rel\}$ is uniformly integrable. With $\mathit{P}$ being finite, uniform integrability implies that for every $X \in \mathcal{X}_l$, we have: $\sup_{t \in \Rel}\int_\Omega |l(-X-t)|d\mathit{P} < \infty$.  
While the integrability condition is not necessary for defining UBSR, we have incorporated it into the definition of $\mathcal{X}_l$ since it is useful for deriving an analytical expression for UBSR, which is useful for risk estimation and optimization. Note that $\Leb^\infty \subset \xl$ holds for any continuous function $l$, and we can extend $\xl$ to contain unbounded random variables by characterizing properties of $l$. The  table below summarizes the sufficiency conditions on $l$ and the corresponding membership of $\xl$.
\begin{table}[h]
    \caption{Sufficiency conditions to construct $\mathcal{X}_l$}
    \label{table:sufficiency_conditions_integrability}
    \centering
    \begin{tabular}{|l||c|c|c|c|} \hline 
         \multirow{2}{*}{Conditions on $l$} &  \multicolumn{4}{|c|}{Membership}\\ \cline{2-5}
         {} &  $\mathcal{L}^\infty \subset \mathcal{X}_l$&  $\mathcal{L}^2 \subset \mathcal{X}_l$&  $\mathcal{L}^1 \subset \mathcal{X}_l$& $\mathcal{L}^0 \subset \mathcal{X}_l$\\ \hline\hline 
         $l$ is continuous&  $\checkmark$&  &  & \\ \hline 
         $l$ is Lipschitz&  $\checkmark$&  $\checkmark$&  & \\ \hline 
         $l$ is concave&  $\checkmark$&  $\checkmark$&  $\checkmark$& \\ \hline 
         $l$ is bounded&  $\checkmark$&  $\checkmark$&  $\checkmark$& $\checkmark$\\ \hline
    \end{tabular}
\end{table}
\Cref{table:sufficiency_conditions_integrability} is to be interpreted as follows: if one wants to define the UBSR for square-summable random variables, then choosing a Lipschitz continuous loss function shall ensure that the uniform integrability condition holds. In the following section, we define UBSR and derive its useful properties.

\section{UBSR for Unbounded Random Variables}
\label{sec:pb}
Let $X \in \xl$ denote the random variable representing the payoff or utility, e.g., of a financial position. Formally, a financial position is a mapping $X: \Omega \to \Rel$, where $X(\omega)$ is the net worth of the position after its realisation. Then, $-X$ denotes the loss incurred, and the UBSR is the smallest amount to be added to $X$ so that the expected loss is below a certain acceptable threshold, say $\lambda$. We formalize the notion of UBSR below. 
\begin{definition}
With the loss function $l$ and a given threshold $\lambda \in \Rel$, the risk measure UBSR of $X \in \xl$ is defined as the function $SR_{l,\lambda}:\xl \to \Rel$, where
\begin{equation}\label{SR}
     SR_{l,\lambda}(X) \triangleq \inf \{ t \in \Rel | \Exp[l(-X-t)] \leq \lambda \} .\\[0.5ex]
\end{equation}    
\end{definition}
Following \cite{artzner-coherent-risk-measures}, we define the \textit{acceptance set} associated with the UBSR risk measure as follows:
\begin{equation}\label{acceptance_set}
    \mathcal{A}_{l,\lambda} = \{ X \in \xl : SR_{l,\lambda}(X) \leq 0\}.
\end{equation}
The set $\mathcal{A}_{l,\lambda}$ contains all acceptable financial positions. For the UBSR in particular, this set contains all random variables $X$ whose expected loss $\Exp[l(-X)]$ does not exceed $\lambda$.

As an example, with $l(x)=\exp(\beta x)$ and $\lambda = 1$,  $SR_{l,\lambda}(X)$ is identical to the entropic risk measure \cite{entropic-risk-measures}, which is a coherent risk measure and enjoys several advantages over the standard risk measures VaR and CVaR. The other popular choice is $l(x)=\frac{x^p}{p} \textrm{ for } x\ge 0$ and $0$ otherwise, with $p>1$.  

 \subsection{Analytical Expression for UBSR}We discuss the problem of quantifying the risk $SR_{l,\lambda}(X)$ of a financial position $X \in \xl$. Note that we do not have an analytical expression for $SR_{l,\lambda}(X)$, and we derive it in the following manner. Consider a real-valued function $g_X:~\Rel~\to~\Rel$ associated with the random variable $X$ as
\begin{equation}
    g_X(t) \triangleq \Exp[ l(-X-t)] - \lambda.
    \label{eq:g}
\end{equation}
Note that $SR_{l,\lambda}(X)$, if it exists, is a root of $g_X(\cdot)$. We now introduce the  following assumptions on $X,l$ and $\lambda$ under which the existence and uniqueness of this root follows.
\begin{assumption}\label{as:1}
There exists $t_{\mathrm{u}}, t_{\mathrm{l}} \in \Rel$ such that $g_X(t_{\mathrm{u}})\leq 0$ and $g_X(t_{\mathrm{l}}) \geq 0$. 
\end{assumption}
\begin{assumption}\label{as:2}
The function $l$ is an increasing function such that $\forall t \in \Rel, \mathit{P}(l'(-X-t) > 0) > 0$.
\end{assumption}   
\begin{assumption}\label{as:3}The function $l$ is continuously differentiable, and the collection of random variables $\{l'(-X-t): t \in R\}$ is uniformly integrable. 
\end{assumption}
The first assumption is required for the existence of UBSR, while the other assumptions ensures that $g_X$ is strictly decreasing, a property we exploit in UBSR estimation. The last assumption is common for interchanging the derivative and the expectation. Similar assumptions have been made in the context of UBSR estimation in \cite{zhaolin2016ubsrest,hegde2021ubsr}. 

The result below establishes that UBSR is the unique root of the function $g_X$ defined in \cref{eq:g}.
\begin{proposition} \label{prop:g}
Let \Crefrange{as:1}{as:3} hold. Then $g_X$ is continuous and strictly decreasing, and the unique root of $g_X$ coincides with $SR_{l,\lambda}(X)$.
\end{proposition}
\begin{remark}
A similar result for bounded random variables has been stated in \cite[Proposition 4.104]{FollmerSchied2004} under the assumption that the loss function $l$ is convex and strictly increasing. We generalize this to unbounded random variables. Unlike \cite{FollmerSchied2004}, our proof does not require the convexity assumption, and we relax the strictly increasing assumption by replacing it with \Cref{as:2}.    
\end{remark}

\subsection{Wasserstein Distance Bounds on UBSR}
We provide results for UBSR estimation and optimization under two different assumptions on the underlying loss functions. These assumptions are specified below.
\begin{assumption}\label{as:l-Lipschitz}
	$l$ is $L_1$-Lipschitz, i.e., there exists $L_1>0$ such that for every $x,y \in \Rel,\,\, \left|l(x) - l(y)\right| \le L_1\left|x - y\right|$.
\end{assumption}
\begin{assumptionp}{\ref*{as:l-Lipschitz}$'$}\label{as:l-smooth}
	$l$ is convex and $L_2$-smooth, and has sub-linear derivative, i.e., there exists $L_2>0,a>0,b>0$ such that for every $x,y \in \Rel$,
	\begin{align}
		\label{eq:l-smooth}l(y)-l(x) - l'(x)(y-x) &\le \frac{L_2}{2}\left(y-x\right)^2, \\
		\label{eq:l-sub-linear-derivative}b \le l'(x) &\le a |x| + b,
	\end{align}
where $b$ simultaneously satisfies \Cref{as:b-lower-bound}. 
\end{assumptionp}   
Using the assumptions above, we derive bounds on the difference between the UBSR values of two random variables $X$ and $Y$. The bounds obtained are in terms of the Wasserstein distance between the corresponding marginal distributions $\mu_X, \mu_Y$. 
Under \Cref{as:l-Lipschitz}, such a result has been obtained in \cite{JMLR-LAP-SPB}. We provide a simpler proof which generalizes to both, Lipschitz and non-Lipschitz loss functions. The Lipschitzness assumption implies that the derivative is bounded, which is often restrictive. \Cref{as:l-smooth} relaxes this condition to allow for unbounded, but sub-linear derivatives. 

In addition, we require the following assumption for all the results.
\begin{assumption}\label{as:b-lower-bound}
    There exists $b>0$ such that for every $x,y \in \Rel, y\geq x \implies l(y)-l(x)\geq b(y-x)$.
\end{assumption}
The result below provides a Wasserstein distance bound for smooth loss functions. This bound will be used in Section \ref{sec:ubsr-est} for establishing error bounds for UBSR estimation.
\begin{lemma}\label{lemma:sr-X-Y-bound}
    Suppose \Cref{as:b-lower-bound} and \Cref{as:l-smooth} hold, and there exists $T>0$ such that $\norm{X}_{\Leb^2} \le T$ for every $X \in \xl$.  Then for every $X,Y \in \xl$, we have
    \begin{align}
        \left| SR_{l,\lambda}(X) - SR_{l,\lambda}(Y) \right|& \nonumber\\
        \label{eq:sr-bound-non-lipschitz}
        \le \frac{L_2}{2b}\mathcal{W}_2^2(\mu_X,\mu_Y) &+ \frac{aT}{b} \mathcal{W}_2(\mu_X,\mu_Y).
    \end{align}
\end{lemma}

\vspace{1ex}

In \cite{JMLR-LAP-SPB}, the authors establish the following bound under \cref{as:l-Lipschitz,as:b-lower-bound}:
    \begin{equation}\label{eq:sr-bound-lipschitz}
        \left| SR_{l,\lambda}(X) - SR_{l,\lambda}(Y) \right| \le \frac{L_1}{b}\mathcal{W}_1(\mu_X,\mu_Y).
    \end{equation}
While the inequality above is useful for deriving bounds for UBSR estimation, \Cref{as:l-Lipschitz} is restrictive since it excludes quadratic losses, which appear naturally in mean-variance optimization. \Cref{as:l-smooth} covers such loss functions, and we bound UBSR difference between two distributions using the $2$-Wasserstein distance in this case.

Next we discuss in brief, the properties that are associated with risk measures. These properties satisfy certain desirable, investor preferences. Readers are referred to \cite{FollmerSchied2004,artzner-coherent-risk-measures} for a detailed study.
\subsection{Convexity of the UBSR Risk Measure}
Suppose $\mathcal{X}$ is an arbitrary set of random variables. We define the notions of monetary and convex risk measures below~\cite{FollmerSchied2004}.
\begin{definition}
    A mapping $\rho : \mathcal{X} \to \Rel$ is called a monetary measure of risk if it satisfies the following two conditions.
\begin{enumerate}
    \item Monotonicity: For all $X_1, X_2 \in \mathcal{X}$ such that $X_1 \leq X_2$ a.s., we have $\rho(X_1) \geq \rho(X_2)$.
    \item Cash invariance: For all $X \in \mathcal{X}$ and $m \in \Rel$, we have $\rho(X + m) = \rho(X) - m$.
\end{enumerate}
\end{definition}
\begin{definition}
A monetary risk measure $\rho$ is convex if it satisfies the following condition for every $X_1,X_2 \in \mathcal{X}$ and $\alpha \in [0,1]$:
\begin{equation}\label{eq:rho_convex}
\rho(\alpha X_1 + (1-\alpha)X_2) \leq \alpha \rho(X_1) + (1-\alpha)  \rho(X_2).    
\end{equation}
\end{definition}
In this paper, we confine our focus to UBSR, which is a convex risk measure. Existing works in \cite{FollmerSchied2004} have shown that UBSR risk measure is convex for the restricted case of bounded random variables ($\mathcal{X} \subset \Leb^\infty$). Using 
 a novel proof technique in the following proposition, we extend the convexity of UBSR to the class $\xl$ of unbounded random variables $x$.
\begin{proposition}\label{proposition:convex_risk_measure}
Suppose \Crefrange{as:1}{as:3} hold for $l$ and every $X \in \xl$. Then $SR_{l,\lambda}(\cdot)$ is a monetary risk measure. If $l$ is also convex, then $\mathcal{A}_{l,\lambda}$ is convex, and $SR_{l,\lambda}(\cdot)$ is a convex risk measure.
\end{proposition}  
This result is useful in showing that the problem of UBSR optimization falls under the class of stochastic convex optimization problems. 

\section{UBSR Estimation}
\label{sec:ubsr-est}
In this section, we discuss techniques to compute UBSR for a given random variable $X$. In practice, the true distribution of $X$ is unavailable, and instead one relies on the samples of $X$ to estimate the UBSR. We use the sample average approximation (SAA) technique \cite{saa-intro-shapiro-alexander,robust-stoc-appx-nemirovski-shapiro} for UBSR estimation of a random variable $X \in \xl$. Such a scheme was proposed and analyzed in \cite{zhaolin2016ubsrest}. We describe this estimation scheme below.
Recall that $SR_{\lambda,l}(X)$ is the solution to the following stochastic problem:
\begin{align}\label{eq:saa_stochastic}
    \text{minimize} \,\,\, t, &\quad
    \text{subject to} \,\,\,\, \Exp[l(-X-t)] \leq \lambda.
\end{align}
Since we do not have access to the true distribution of $X$, instead of solving \cref{eq:saa_stochastic} we use $m$ i.i.d samples $\{Z_i\}_{i=1}^m$ (also denoted as a random vector $\mathbf{Z}$) from $X$ to solve the following  problem to estimate $SR_{\lambda,l}(X)$:
\begin{align}\label{eq:saa_deterministic}
    \text{minimize} \,\,\, t, &\quad
    \text{subject to} \,\,\,\,\, \frac{1}{m} \sum_{i=1}^{m}l(-Z_i-t) \leq \lambda.
\end{align}
Define the function $SR_{l,\lambda}^m:\Rel^m \to \Rel$ as follows:
\begin{equation}
    SR^m_{l,\lambda}(z) \triangleq \min \left\{ t \in \Rel \left| \frac{1}{m} \sum_{j=1}^m l(-z_j-t) \leq \lambda \right. \right\}.
\end{equation}
$SR^m_{l,\lambda}(\mathbf{Z})$ is the solution to \cref{eq:saa_deterministic} and an estimator of $SR_{l,\lambda}(X)$. For the analysis of this UBSR estimator, we make the following assumption on the random variable $X$:
\begin{assumption}\label{as:X_bounded_moment}
$\exists q \geq 2, T > 0$ such that $\norm{X}_{\Leb^q}$ is finite and $\norm{X}_{\Leb^2}\leq T$.
\end{assumption}
Let $\mu$ denote the distribution of $X$, and let $\mu_m$ be its $m$-sample empirical distribution. Let $SR_m(\mathbf{Z})$ denote the SAA-estimate of $SR_{l,\lambda}(X)$. The  result below presents error bounds for the UBSR estimator in \cref{eq:saa_deterministic}. 
\begin{lemma}\label{lemma:sr-x-y-exp-1}
    Suppose \Cref{as:X_bounded_moment} holds for $q>4$ and \cref{as:l-Lipschitz,as:b-lower-bound} hold. Then, we have
    \begin{align*}
        \Exp\left[\left| SR_{l,\lambda}(X) - SR_m(Z) \right|\right] &\le \frac{C_1}{\sqrt{m}}, \,\, \text{ and }\\
        \Exp\left[\left| SR_{l,\lambda}(X) - SR_m(Z) \right|^2\right] &\le \frac{C_2}{\sqrt{m}},
    \end{align*}
    where $C_1=\frac{48L_1T}{b}$ and $C_2=\frac{108L_1^2T^2}{b^2}$ respectively.
\end{lemma}
The result below is a variant of Lemma \ref{lemma:sr-x-y-exp-1} that caters to smooth loss functions.
\begin{lemmap}{\ref{lemma:sr-x-y-exp-1}\large\textquotesingle}\label{lemma:sr-x-y-exp-2}
    Suppose \Cref{as:X_bounded_moment} holds for $q>8$ and \cref{as:l-smooth,as:b-lower-bound} hold. Then we have
    \begin{align*}
        \Exp\left[\left| SR_{l,\lambda}(X) - SR_m(Z) \right|\right] &\le \frac{C_3}{m^{1/4}}, \text{ and } \\
        \Exp\left[\left| SR_{l,\lambda}(X) - SR_m(Z) \right|^2\right] &\le \frac{C_4}{\sqrt{m}},
    \end{align*}
    where $C_3=\frac{(27L_2 + 8 a)T^2}{b}$ and $C_4=\frac{(540L_2^2+108 a^2)T^4}{b^2}$.
\end{lemmap}
    
    The constants $C_1,\ldots, C_4$ appearing in the lemmas above are derived based on a result in \cite{fournier-2023}, and depend inversely on $q$, i.e., assuming a higher moment bound leads to lower constants. 
    
While computing $SR^m_{\lambda,l}(\mathbf{Z})$ requires solving a convex optimization problem, a closed form expression is not available for any given loss function. 
Instead, it is possible to obtain an estimator within $\delta$-neighbourhood of $SR^m_{\lambda,l}(\mathbf{Z})$. The following lemma extends the bounds from \Cref{lemma:sr-x-y-exp-1} to such an estimator.
\begin{proposition}\label{proposition:sr-estimator} Suppose \crefrange{as:1}{as:X_bounded_moment} hold. Given any $\delta>0$, let $\hat{t}_m$ be an approximate solution to \cref{eq:saa_deterministic}, that is in the $\delta$-neighbourhood of $SR^m_{\lambda,l}(\mathbf{Z})$. Then, with the choice $\delta = \frac{d_1}{\sqrt{m}}$ for some $d_1>0$, we have the following.
    \begin{align*}
        \Exp[|\hat{t}_m - SR_{\lambda,l}(X)|] &\leq \frac{d_1+C_1}{\sqrt{m}}, \,\,\, \text{ and } \\
        \Exp[\left(\hat{t}_m - SR_{\lambda,l}(X)\right)^2] &\leq \frac{2(d_1^2+C_2)}{m}.
    \end{align*}
    where, $C_1,C_2$ are given in \cref{lemma:sr-x-y-exp-1}.
\end{proposition}
\begin{remark}
    Corresponding result for the non-Lipschitz case follows trivially by replacing \cref{as:l-Lipschitz} with \cref{as:l-smooth} and choosing $\delta=\frac{d_1}{m^{1/4}}$. \footnote{For the rest of the paper, we shall write separate results for Lipschitz and non-Lipschitz case only when the corresponding proofs are different.}
\end{remark}
We now describe a variant of the bisection method to compute the solution to the optimization problem \cref{eq:saa_deterministic}. 
Algorithm \ref{alg:saa_bisect} presents the pseudocode to get $\hat{t}_m$ that satisfies \Cref{proposition:sr-estimator}. 
\begin{algorithm}
\SetKwInOut{Input}{Input}\SetKwInOut{Output}{Output}
\SetKwInOut{Define}{Define}
\SetAlgoLined
\Input{$\delta > 0, \text{ i.i.d. samples } \{Z_i\}_{i=1}^m$}
\Define{$\hat{g}(t) \triangleq \frac{1}{m} \sum_{i=1}^m l(-Z_i - t) - \lambda$}
 $low \leftarrow \min(0, sign(\hat{g}(0)))$\; 
 $high \leftarrow \max(0, sign(\hat{g}(0)))$\; 
 \lWhile{$\hat{g}(high) > 0$}{$high \leftarrow 2 * high$}
 \lWhile{$\hat{g}(low) < 0$}{$low \leftarrow 2 * low$}
 $T \leftarrow high$ - $low, \hat{t} \leftarrow (low+high)/2$\;
 \While{$T > 2\delta$}{
  \leIf{$\hat{g}(\hat{t} )>0$}{$low \leftarrow \hat{t} $}{$high \leftarrow \hat{t}$}
  $T \leftarrow high$ - $low, \hat{t} \leftarrow (low+high)/2$\;
 \caption{UBSR-SB (Search and Bisect)}
 }
\Output{$\hat{t}$}
 \label{alg:saa_bisect}
\end{algorithm}

In \cite{zhaolin2016ubsrest}, the authors assume knowledge of $t_{\mathrm{l}},t_{\mathrm{u}}$ for a bisection method to solve \cref{eq:saa_deterministic}.  However, these values are seldom known in practice. Our algorithm does not require $t_{\mathrm{l}}, t_{\mathrm{u}}$. Instead, the algorithm works by finding the search interval first, and then performing a bisection search. $low,high$ are proxies for $t_{\mathrm{l}},t_{\mathrm{u}}$ and the first two loops of the algorithm find $low,high$ such that $SR^m_{l,\lambda}(\mathbf{Z}) \in [low, high]$. The final loop in the algorithm performs bisection search to return a value in the $\delta$-neighbourhood of $SR^m_{l,\lambda}(\mathbf{Z})$.

We now analyze the iteration complexity of Algorithm \ref{alg:saa_bisect}. Suppose the first and second loops run for $n_1,n_2$ iterations. It is trivial to see that $n_1 < 1 + \log_2(|t_{\mathrm{u}}|)$ and $n_2 < 1 + \log_2(|t_{\mathrm{l}}|)$. Due to the carefully chosen initial values of variables $low$ and $high$, either $n_1$ or $n_2$ will always be $0$. Then $T \leq 2^n$ holds, where $n \triangleq \max(n_1, n_2)$. Suppose the final loop terminates after $k$ iterations. Then at $k-1$, we have $\frac{T}{2^{k-1}} > 2\delta$ which implies that $k < 1 + \log_2\left(\frac{\max(|t_{\mathrm{u}}|,|t_{\mathrm{l}}|)}{\delta}\right)$.
Thus the total iteration complexity of the algorithm is at most $\max(n_1,n_2) + k $ which is upper-bounded by $ 2\left(1+ \log\left( \frac{\max(|t_{\mathrm{u}}|, |t_{\mathrm{l}}|)}{\sqrt{\delta}} \right)\right)$.



\section{UBSR Optimization}
\label{sec:ubsr-opt}
Let $\Theta \subset \Rel^d$ be a compact and convex set. Given a function $F: \Theta \times \Rel \to \Rel$ and a random variable $\xi$, a standard stochastic optimization algorithm deals with the problem of minimizing $\Exp[F(\theta, \xi)]$ using samples of either $F(\theta, \xi)$ or $\nabla_\theta F(\theta, \xi)$. Instead, we are interested in the problem of minimizing the UBSR of $F(\theta,\xi)$, i.e., to find a
\begin{align}\label{eq:hthetaopt}
\theta^* &\in \argmin_{\theta \in \Theta} h(\theta),  
\text{ where } h(\theta) \triangleq SR_{l,\lambda}(F(\theta, \xi)).
\end{align}
In our setting, we can obtain samples of $\xi$, which can be used to compute $F(\theta, \xi)$ and $\nabla_\theta F(\theta, \xi)$. Under the assumption that for every $\theta \in \Theta, F(\theta, \xi) \in \xl$, we  express $h:\Theta \to \Rel$ as follows\footnote{For notational convenience, we suppress the dependency of $l,\lambda$ on $h$.}: 
\begin{equation*}
     h(\theta) = \inf \left\{ t \in \Rel \left| \Exp[l(-F(\theta,\xi)-t)] \leq \lambda \right.\right\} .
\end{equation*}    
In the next section, we derive properties of UBSR and establish conditions under which $h$ is strongly convex. Subsequently, in \Cref{ss:ubsr-gradient-theorem}, we derive an expression for the gradient of $h(\theta)$, and use this expression to arrive at a biased gradient estimator for $\nabla h$ in \Cref{ss:ubsr_grad_estimator}. In \Cref{sec:ubsr-sgd}, we incorporate this gradient estimator into an SG scheme for solving \cref{eq:hthetaopt}. We provide non-asymptotic bounds for this SG algorithm for the case when $h$ is strongly convex.

\subsection{Properties of UBSR}
\label{ss:ubsr-properties}
For the sake of readability, we restate the assumptions and propositions that were derived in \Cref{sec:pb}, in terms of $h$ and $\theta$. 
We first define $g:\Rel \times \Theta \to \Rel$ as follows:
\begin{equation}\label{eq:g_t_theta}
    g(t,\theta) \triangleq \Exp[ l(-F(\theta,\xi)-t)] - \lambda.
\end{equation}
For deriving the UBSR gradient expression, we require \Crefrange{as:1}{as:3} to hold for every $\theta \in \Theta$. We restate these assumptions below.
\begin{assumption}\label{as:1_with_theta}
For every $\theta \in \Theta$, there exists $t_u(\theta), t_l(\theta) \in \Rel$ such that $g(t_u(\theta), \theta) \leq 0$ and $g(t_l(\theta),\theta) \geq 0$. 
\end{assumption}
\begin{assumption}\label{A_increasing}
$l$ is an increasing function such that for every $\theta \in \Theta$ and $t \in \Rel$,
\begin{equation*}
    \mathit{P}(l'(-F(\theta,\xi)-t) > 0) > 0.
\end{equation*}
\end{assumption}   
\begin{assumption} \label{as:3_with_theta} $l$ is continuously differentiable, and for every $\theta \in \Theta$, the set of random variables $\{l'(-F(\theta,\xi)-t): t \in R\}$ is uniformly integrable. 
\end{assumption}
\Cref{prop:g} is restated as follows:
\begin{proposition}\label{P_main}
Suppose that \Crefrange{as:1_with_theta}{as:3_with_theta} hold. Then for every $\theta \in \Theta$, $g(\cdot, \theta)$ is continuous, strictly decreasing, and has a unique root that coincides with $h(\theta)$, i.e., for every $\theta \in \Theta$,
\begin{equation}
    g(h(\theta), \theta) = 0.
\end{equation}    
\end{proposition}
We require the result above to invoke the implicit function theorem in the derivation of the UBSR gradient expression in the following section.

The result below shows that $h$ is strongly convex under conditions similar to those in \Cref{proposition:convex_risk_measure}, except that we impose a strong concavity requirement on $F$. These assumptions are satisfied in the context of a mean-variance portfolio optimization problem.
\begin{lemma}\label{lemma:h-convex}
    Let $l$ be convex and let $F(\cdot,\xi)$ be $\mu$-strongly concave w.p. $1$. Let \Cref{as:1_with_theta,as:3_with_theta} hold, then $h$ is $\mu$-strongly convex.
\end{lemma}

\subsection{UBSR gradient expression}\label{ss:ubsr-gradient-theorem}
For deriving the UBSR gradient expression, we require the following assumption:
\begin{assumption}\label{as:f_c_diff}
    $F(\cdot, \xi)$ is continuously differentiable almost surely. 
\end{assumption}
An assumption similar to the above has been made earlier in \cite{zhaolin2016ubsrest,hegde2021ubsr} in the context of bounded random variables.
 
The following lemma uses dominated convergence theorem to interchange the derivative and the expectation in \cref{eq:g_t_theta}, and gives the expressions for the partial derivatives of $g$.
\begin{lemma}\label{lemma:dct}
    Suppose \Crefrange{as:3_with_theta}{as:f_c_diff} hold. Then $g$ is continuously differentiable on $\Rel \times \Theta$, and the partial derivatives are given by
    \begin{align}\label{part_der}
    \frac{\partial g(t,\theta)}{\partial \theta_i} &= - \Exp\left[ l'(-F(\theta,\xi)-t) \frac{\partial F(\theta, \xi)}{\partial \theta_i} \right],  \\ \label{part_der2}
    \frac{\partial g(t,\theta)}{\partial t} &= - \Exp\big[ l'(-F(\theta,\xi)-t) \big],
    \end{align}
    where $i \in \{1,2,\ldots,d\}$.
\end{lemma}
The next lemma shows that the partial derivatives of $h$ can be expressed as a ratio of partial derivatives of $g$, using the implicit function theorem in conjunction with the results derived in \Cref{lemma:dct} and \Cref{P_main}. 
\begin{lemma}\label{lemma:ift}Let \Crefrange{as:1_with_theta}{as:f_c_diff} hold. Then, $h$ is a continuously differentiable mapping on $\Theta$ and its partial derivatives can be expressed as:
\begin{equation}\label{eq:implicit}
    \frac{\partial h(\theta)}{\partial
    \theta_i}\Big|_{\tilde{\theta}} = -  \frac{\partial g(t,\theta) / \partial \theta_i}{\partial g(t,\theta) / \partial t}\Big|_{h(\tilde{\theta}), \tilde{\theta}} \qquad \forall i \in \{1,2,\ldots,d\}.
\end{equation}
\end{lemma}
We now present the main result that provides an expression for the gradient of UBSR. 
\begin{theorem}[Gradient of UBSR]\label{theorem:ubsr-gradient}
Suppose \Cref{as:1_with_theta,A_increasing,as:3_with_theta,as:f_c_diff} hold. Then the function $h$ is continuously differentiable and the gradient of $h$ can be expressed as follows: 
\begin{equation}\label{h_grad_0}
\nabla h(\theta) = - \frac{\Exp \Big[ l'(-F(\theta,\xi)-h(\theta)) \nabla F(\theta,\xi) \Big]}{\Exp \big[ l'(-F(\theta,\xi)-h(\theta)) \big]}. 
\end{equation}
\end{theorem}
\begin{corollary}
    The function $h$ is Lipschitz, i.e., there exists $L_5>0$ such that following holds. 
    \begin{align}\label{eq:h-lipschitz}
    |h(\theta_1) - h(\theta_2)| \leq L_3 \norm{\theta_1 - \theta_2}_2.
\end{align}
\end{corollary}
The corollary follows from the theorem since $\Theta$ is compact. The constant $L_3$ depends on the functions $\mathit{F}$ and $l$.  

\subsection{UBSR gradient estimation}\label{ss:ubsr_grad_estimator}
From \cref{h_grad_0}, it is apparent that an estimate of $h(\theta)$ is required to form an estimate for $\nabla h(\theta)$. For estimating $h(\theta)$, we employ the scheme presented in \Cref{sec:ubsr-est}. Suppressing the dependency on $l,\lambda$, we define the functions $SR^m_{\theta}:\Rel^m \to \Rel$ and $J^m_\theta:\Rel^m \times \Rel^m \to \Rel^d$ as follows:
\begin{align}\label{eq:SR-definition}
    SR^m_{\theta}(z) &\triangleq \min \left\{ t \in \Rel \left| \frac{1}{m} \sum_{j=1}^m l(-F(\theta,z_j)-t) \leq \lambda \right. \right\}, \\
     \label{eq:J-definition}
     J^m_\theta(z,\hat{z}) &\triangleq \frac{\sum_{j=1}^m \Big[ l'(-F(\theta,z_j)-SR^m_{\theta}(\hat{z})) \nabla F(\theta,z_j)\Big]}{\sum_{j=1}^m \big[ l'(-F(\theta,z_j)-SR^m_{\theta}(\hat{z})) \big]}.
\end{align}
Given a $\theta \in \Theta$, we use $SR^m_{\theta}(\cdot),J^m_\theta(\cdot,\cdot)$ to estimate $h(\theta),\nabla h(\theta)$ respectively. If $z,\hat{z}$ are constructed using i.i.d. samples of $\xi$, then one can obtain tight bounds on the estimation error, as derived in the following section.

To derive error bounds for \cref{eq:J-definition}, it is necessary to bound the distance between UBSR estimate $SR^m_{\theta}$ and the true UBSR $h(\theta)$. This bound can be obtained from \Cref{lemma:sr-x-y-exp-1} by replacing \Cref{as:X_bounded_moment} with the \Cref{as:f_theta_bounded_moment} below. To that end, we restate \Cref{lemma:sr-x-y-exp-1} as \Cref{lemma_srh} below for the sake of readability.
\begin{assumption}\label{as:f_theta_bounded_moment}
    There exists $q\geq 2$ and $T>0$ such that for every $\theta \in \Theta$, $\norm{F(\theta,\xi)}_{\Leb^q}$ is finite and $ \norm{F(\theta,\xi)}_{\Leb^2} \leq T$. 
\end{assumption}
\begin{lemma}[UBSR estimation bounds]\label{lemma_srh}
    Suppose \Cref{as:l-Lipschitz,as:b-lower-bound,as:f_theta_bounded_moment} hold and $\mathbf{Z}$ be an $m$-dimensional random vector such that each $Z_j$ is an independent copy of $\xi$. Then,
    \begin{align}
        \Exp[|SR^m_\theta(\mathbf{Z}) - h(\theta)|] \!\leq \!\frac{C_1}{\sqrt{m}},\,\, 
        \Exp\left[\left[SR^m_\theta(\mathbf{Z}) - h(\theta)\right]^2\right] \!\leq \!\frac{C_2}{m},
    \end{align}
    where $C_1,C_2$ are given in \Cref{lemma:sr-x-y-exp-1}.
\end{lemma}
We  make the following assumptions on $F$ and $\nabla F$. 
\begin{assumption}\label{as:lp_moments_F_dF}
    Suppose $p \in [1,\infty]$ is such that the $p^{th}$ moments of $F, \nabla F$ exist and are finite, and there exists $M>0L_3>0,L_4>0$ such that for every $\theta \in \Theta$, we have
    \begin{equation*}
        \norm{\nabla F(\theta, \xi)}_{\Leb^p} \leq M,
    \end{equation*}
    and for every $\theta_1,\theta_2 \in \Theta$,
    \begin{align*}
        \norm{F(\theta_1,\xi) - F(\theta_2,\xi)}_{\Leb^p} &\leq L_4 \norm{\theta_1-\theta_2}_p, \\
        \norm{\nabla F(\theta_1, \xi)- \nabla F(\theta_2,\xi)}_{\Leb^p} &\leq L_5 \norm{\theta_1-\theta_2}_p.
    \end{align*}
\end{assumption}
Similar assumptions have been made before in \cite{hegde2021ubsr} for the non-asymptotic analysis of UBSR optimization scheme. Under these assumptions we establish that the objective function $h(\cdot)$ is smooth in the next section.

We now derive error bounds on the gradient estimator $J^m_\theta$, under  the following bounded variance assumption:
\begin{assumption}\label{sigma_bound} 
There exist $\sigma_1, \sigma_2>0$ such that the following bounds hold for every $\theta \in \Theta$ and $i \in \{1,2,\ldots,d\}$:
\begin{align*}
&\mathbf{Var}(l'(-F(\theta,\xi)-h(\theta))\frac{\partial F(\theta,\xi)}{\partial \theta_i} ) \leq \sigma_1^2, \\
&\mathbf{Var}(l'(-F(\theta,\xi)-h(\theta))) \leq \sigma_2^2. 
\end{align*}
\end{assumption}
An assumption that bounds the variance of the gradient estimate is common to the non-asymptotic analysis of stochastic gradient algorithms, cf. \cite{moulines2011non,bottou2018optimization,bhavsar2021nonasymptotic}. The assumption above is made in a similar spirit, and the result below establishes that the mean-squared error of the UBSR gradient estimate \cref{eq:J-definition}  vanishes asymptotically at a $O(1/m)$ rate. 
\begin{lemma}[UBSR gradient estimator bounds] \label{lemma:grad_est_bound}
    Suppose \Crefrange{as:1_with_theta}{sigma_bound} hold. Let $\mathbf{Z,\hat{Z}}$ denote $m$-dimensional random vectors such that each $Z_j$ and each $\hat{Z}_j$ are i.i.d copies of $\xi$, and, $\mathbf{Z}$ and $\mathbf{\hat{Z}}$ are independent. Then for every $\theta \in \Theta$, the gradient estimator $J^m_\theta(\mathbf{Z,\hat{Z}})$ satisfies 
    \begin{align*}
        \Exp \left[ \norm{J^m_\theta(\mathbf{Z,\hat{Z}}) - \nabla h(\theta)}_1 \right] &\leq \frac{D_1}{\sqrt{m}}, \\
        \text{and} \quad \Exp \left [\norm{J^m_\theta(\mathbf{Z,\hat{Z}}) - \nabla h(\theta)}_2^2 \right] &\leq \frac{D_2}{m},
    \end{align*}
    where $D_1 = \frac{2C_1L_1L_2M+L_1(\sigma_1\sqrt{d}+M\sigma_2)}{b^2}$  and $D_2 = \frac{8C_2L_1^2L_2^2M^2+4L_1^2(d\sigma_1^2+M^2\sigma_2^2)}{b^4}$.
\end{lemma}
The above result shows that to get $\epsilon$-accurate gradient estimate, one requires samples of the order of $\mathcal{O}(1/\epsilon^2)$. 

The gradient estimator in \cref{eq:J-definition} uses $SR^m_{l,\lambda}(\theta)$ from \cref{eq:SR-definition} in both numerator and denominator as an estimate of the UBSR $h(\theta)$. As identified in \Cref{sec:ubsr-est}, we cannot compute the exact value of $SR^m_{l,\lambda}(\theta)$. Instead, we use a gradient estimator where we replace $SR^m_{l,\lambda}(\theta)$ with its approximate value given by the Algorithm \ref{alg:saa_bisect}. The result below provides an error bound for this gradient estimator.
\begin{proposition}\label{prop:grad-est-bound-algo-1}
    Let $\hat{J}^m_\theta(\mathbf{Z})$ be the gradient estimator constructed by replacing $SR^m_\theta(\mathbf{\hat{Z}})$ in \cref{eq:J-definition} with its approximation obtained from Algorithm \ref{alg:saa_bisect} using $\delta=\frac{d_1}{\sqrt{m}}$ for some $d_1>0$. Then we have the following error bound on this gradient estimator:
    \begin{align*}
        \Exp \left[ \norm{\hat{J}^m_\theta(\mathbf{Z}) - \nabla h(\theta)}_1 \right] &\leq \frac{\hat{D}_1}{\sqrt{m}}, \\
        \text{and} \quad \Exp \left [\norm{\hat{J}^m_\theta(\mathbf{Z}) - \nabla h(\theta)}_2^2 \right] &\leq \frac{\hat{D}_2}{m},
    \end{align*}
    where $\hat{D}_1 = \frac{2(C_1+d_1)L_1L_2M+L_1(\sigma_1\sqrt{d}+M\sigma_2)}{b^2}$  and \\$\hat{D}_2 = \frac{8(C_2+d_1^2)L_1^2L_2^2M^2+4L_1^2(d\sigma_1^2+M^2\sigma_2^2)}{b^4}$.
\end{proposition}
\begin{remark}
    The authors in \cite{zhaolin2016ubsrest} consider the scalar case, and show that the UBSR derivative estimator is  asymptotically consistent, while the authors in \cite{hegde2021ubsr} establish non-asymptotic error bounds for this estimator. Our result generalizes to the multivariate case, and we provide non-asymptotic error bounds for the UBSR gradient estimator, which can be used to infer asymptotically consistency.
\end{remark}
\subsection{SG Algorithm for UBSR optimization}
\label{sec:ubsr-sgd}
\begin{algorithm}
\SetKwInOut{Input}{Input}\SetKwInOut{Output}{Output}
\SetAlgoLined
\Input{$\theta_0 \sim \Theta$, thresholds $\{\delta_k\}_{k \ge 1}$, batch-sizes $\{m_k\}_{k \ge 1}$ and step-sizes $\{\alpha_k\}_{k \ge 1}$.}
\For{$k= 1, 2, \ldots, n$}{
    sample $\mathbf{\hat{Z}}^k = [\hat{Z}^k_1, \hat{Z}^k_2, \ldots, \hat{Z}^k_{m_k} ]$\; 
    compute $t^k$ using $\delta_k,\mathbf{\hat{Z}}^k$ in Algorithm \ref{alg:saa_bisect}\;
    sample $\mathbf{Z}^k = [Z^k_1, Z^k_2, \ldots, Z^k_{m_k} ]$\; 
    compute $J^k = \hat{J}_{\theta_{k-1}}^{m_k}(\mathbf{Z}^k)$ using $t^k$\;
    update $\theta_k \leftarrow \Pi_\Theta\left(\theta_{k-1} - \alpha_k J^k\right)$\;
}
\caption{UBSR-SG}
\Output{$\theta_n$}
 \label{alg:ubsr-minimization}
\end{algorithm}
Algorithm \ref{alg:ubsr-minimization} presents the pseudocode for the SG algorithm to optimize UBSR. 
In this algorithm, $\Pi_{\Theta}(x) \triangleq \arg \min_{\theta \in \Theta} 
 \norm{x - \theta}_2$ denotes the operator that projects onto the convex and compact set $\Theta$.  In each iteration $k$ of this algorithm, we sample $m_k$-dimensional random vectors $\mathbf{Z^{k},\hat{Z}^{k}}$ that are independent of one another and independent of the previous samples, such that for every $ i \in [1,2,\ldots,m_k], Z_i^{k} \sim \xi, \hat{Z}_i^{k} \sim \xi$, and then perform the update given in the Algorithm \ref{alg:ubsr-minimization}, with a arbitrarily chosen $\theta_0 \in \Theta$. 
 
 Let $\mathcal{F}_0=\sigma(\theta_0)$ and for every $k \in \mathbb{N}$, let $\mathcal{F}_k=\sigma\left(\theta_0,\mathbf{Z^{1},\hat{Z}^{1},\ldots,Z^{k},\hat{Z}^{k}}\right)$. Then $\{\mathcal{F}_k\}_{k \geq 0}$ forms the filtration that $\theta_n$ is adapted to. Applying the independence lemma from \cite{shreve2005stochastic} with \Cref{prop:grad-est-bound-algo-1}, we have $\forall k \in \mathbb{N}$,
\begin{align*}
    \Exp\left[ \left. \norm{\hat{J}^{m_k}_\theta(\mathbf{Z}) - \nabla h(\theta_{k-1})}_1  \right| \mathcal{F}_{k-1} \right] &\leq \frac{\hat{D}_1}{\sqrt{m_k}}, \\
    \Exp\left[ \left. \norm{\hat{J}^{m_k}_\theta(\mathbf{Z}) - \nabla h(\theta_{k-1})}_2^2  \right| \mathcal{F}_{k-1} \right] &\leq \frac{\hat{D}_2}{m_k}.
\end{align*}
 For the algorithm presented above, we derive non-asymptotic bounds for two choices of the batch-sizes, namely constant and increasing. We assume the objective function $h(\cdot)$ is strongly convex. Such an assumption is satisfied in a portfolio optimization example with Gaussian rewards (see \cref{sec:portfolio-expts} for the details). 
\begin{assumption}\label{as:strong-convexity}
    $h$ is $\mu$-strongly convex, i.e.,  $\forall \theta \in \Theta, \nabla^2 h(\theta) - \mu I \succeq 0$.
\end{assumption}
\begin{theorem}\label{theorem:sgd_convergence}
    Suppose \Cref{as:l-Lipschitz,as:b-lower-bound} and \Crefrange{as:1_with_theta}{as:strong-convexity} hold, and let $\theta^*$ be the minimizer of $h(\cdot)$. Let $\alpha_k = \frac{c}{k}$, for some $c>0$ and let $c$ satisfy the condition : $1 \leq \mu c \leq 3$. If $m_k=k,\,\forall k$, then 
    \begin{align*}
        \nonumber &\Exp \left[ \norm{\theta_n - \theta^*}_2^2 \right]
        \leq \frac{512c^2\hat{D}_1^2}{n+1} \\ &+\frac{450\Exp\left[\norm{\theta_0 - \theta^*}_2^2\right]+64c^2\hat{D}_2\left(1+\log(n)\right)}{(n+1)^2},
    \end{align*}
    where $\hat{D}_1,\hat{D}_2$ are as defined in \Cref{prop:grad-est-bound-algo-1}.

    In addition, if $m_k = m$ for all $k$, then we have
    \begin{align*}
        &\Exp \left[ \norm{\theta_n - \theta^*}_2^2 \right] \leq \frac{64c^2\hat{D}_1^2}{m} +\frac{450\Exp\left[\norm{\theta_0 - \theta^*}_2^2\right]+64c^2\hat{D}_2}{(n+1)m}.
    \end{align*}
\end{theorem}
\begin{corollary}\label{corr:h-smooth}
	Under conditions of \Cref{theorem:sgd_convergence},
    $h$ is $L_6$-smooth, and 
    \begin{equation*}
        \Exp \left[ h(\theta_n) - h(\theta^*)\right] \leq L_6 \Exp \left[ \norm{\theta_n - \theta^*}_2^2 \right],
    \end{equation*}
    where $L_6 =  \frac{L_2 M\left(\frac{1+L_1(L_3+L_4)}{b}\right) + L_1 L_5}{b}$. 
\end{corollary}
\begin{remark}
    Asymptotic convergence rate of $\mathcal{O}(1/n)$ has been derived earlier in \cite{zhaolin2016ubsrest} for the scalar UBSR optimization case, but their result required a batchsize $m\geq n$ for each iteration. Our result not only establishes a non-asymptotic bound of the same order, but also allows for an increasing batchsize that does not depend on $n$. \Cref{table:complexity-analysis} summarizes the convergence rates for different choices of batch-sizes.
\end{remark}
\begin{remark}
    For the case of non-Lipschitz loss function, we can get non-asymptotic convergence rate of $\mathcal{O}(1/ \sqrt{n})$ by replacing \cref{as:l-Lipschitz} with \cref{as:l-smooth}. We do not state it as a separate result since the proof remains identical to that of \Cref{theorem:sgd_convergence}.
\end{remark}
\begin{table}[ht]
    \caption{Complexity bounds for UBSR-SG to ensure $\Exp\left[h(\theta_n) - h(\theta^*)\right] \leq \epsilon$.}
    \label{table:complexity-analysis}
    \centering
    \begin{tabular}{|c||c|c|} \hline 
         Batchsize &  $m_k=k$&$m_k=n^p$\\ \hline\hline 
         Iteration complexity &   $\mathcal{O} \left(1/\epsilon\right)$&$\mathcal{O} \left(1/\epsilon^\frac{1}{p}\right)$\\ \hline 
         Sample complexity &   $\mathcal{O} \left(1/\epsilon^2\right)$&$\mathcal{O} \left(1/\epsilon^{1+\frac{1}{p}}\right)$\\ \hline
    \end{tabular}
\end{table}

\section{Simulation Experiments}
\label{sec:experiments}
In this section, we  demonstrate the effectiveness of the two proposed algorithms UBSR-SB and UBSR-SG in solving estimation and optimization problems respectively. 
\subsection{UBSR Estimation.}In this experiment, we use samples from a standard normal random variable $X$ and estimate $SR_{l,\lambda}(X)$ with the following UBSR parameters: loss function $l(x) \triangleq \exp{(\beta x)}$ with $\beta=0.5$, and risk tolerance $\lambda = 0.05$. Such an exponential loss function is common in risk estimation, cf. \cite{zhaolin2016ubsrest,dunkel2010stochastic,hegde2021ubsr}. This choice of parameters allows for the following closed form expression for UBSR. 
\begin{equation}\label{eq:sr-l-exp-x-normal}
     t^* \triangleq SR_{l,\lambda}(X) = -\mu + \frac{\beta \sigma^2}{2} - \frac{\log(\lambda)}{\beta}.
\end{equation}
\begin{figure}[thpb]
    \centering
    \framebox{
    \includegraphics[width=8cm, height=5cm]{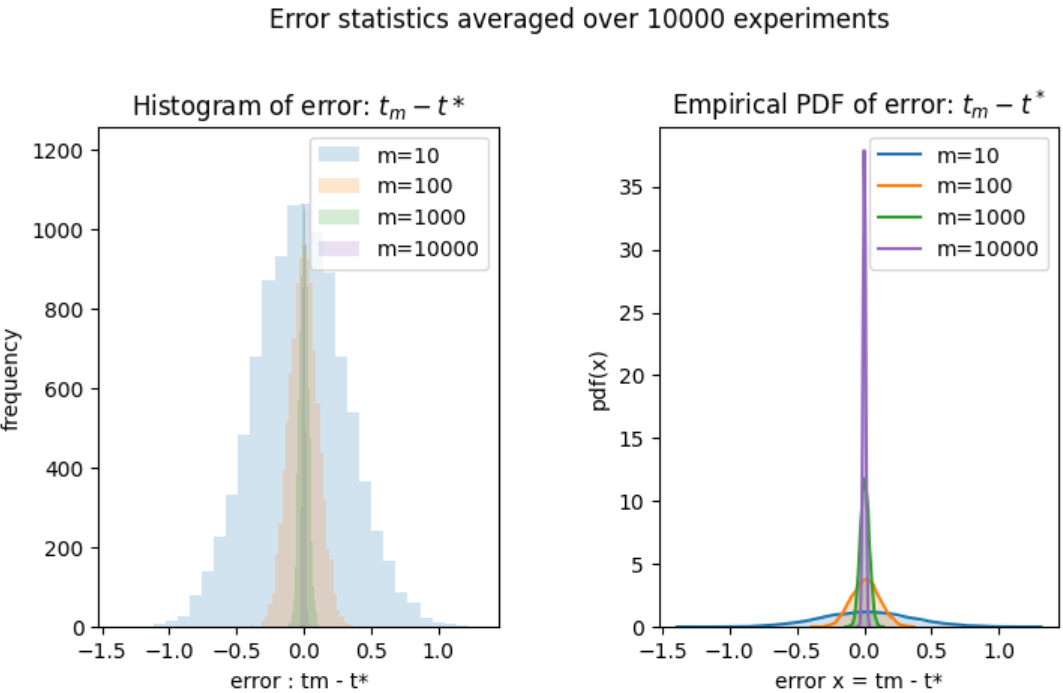}}
    \caption{Error distribution of the $m$-sample estimate, $t_m$ obtained using algorithm UBSR-SB, for different choices of $m$.}
    \label{fig:ubsr-sb-distribution}
\end{figure}

For a given sample-size $m$, we run the experiment $10000$ times. In each run, we obtain $m$ samples from $X$ and get a SAA estimate $t_m$ via algorithm UBSR-SB. We plot the error $t_m-t^*$ in \Cref{fig:ubsr-sb-distribution} for different choices of $m$. Both, the mean and variance of this error, vanish as the number of samples $m$ increases.
\subsection{Portfolio Optimization}
\label{sec:portfolio-expts}
Suppose we have a set of $d$ assets in a financial market. Let the random vector $\mathbf{\xi} \in \Rel^d$ denote asset-wise market returns. Given an asset allocation $\theta \in \Theta$, the random variable $F(\theta, \mathbf{\xi})\triangleq \xi^T\theta$ denotes portfolio returns. With $h(\theta) \triangleq SR_{l,\lambda}(F(\theta,\xi))$, one can obtain risk-optimal portfolio allocation using the UBSR-SG algorithm. 

We demonstrate the convergence of UBSR-SG for the Markowitz model of portfolio optimization. In this model, $\xi$ follows a multivariate normal distribution with mean $\mu$ and covariance $\Sigma$, where $\Sigma$ is positive-definite. We choose $l(x)=\exp{(\beta x)}$ with $\beta=0.5$, and risk tolerance $\lambda = 0.05$. 

Since $\xi \sim N(\mu, \Sigma),$ we have $F(\theta,\xi) \sim N\left(\mu^T \theta, \theta^T\Sigma\theta\right)$. From \cref{eq:sr-l-exp-x-normal}, we have the following expression for UBSR:
\begin{equation*}\label{eq:h-quadratic}
    h(\theta) = -\mu^T \theta + \frac{\beta\left[\theta^T\Sigma\theta\right]}{2} - \frac{\log(\lambda)}{\beta}.
\end{equation*}

\begin{figure}[thpb]
      \centering
      \framebox{
      \includegraphics[width=6cm, height=6cm]{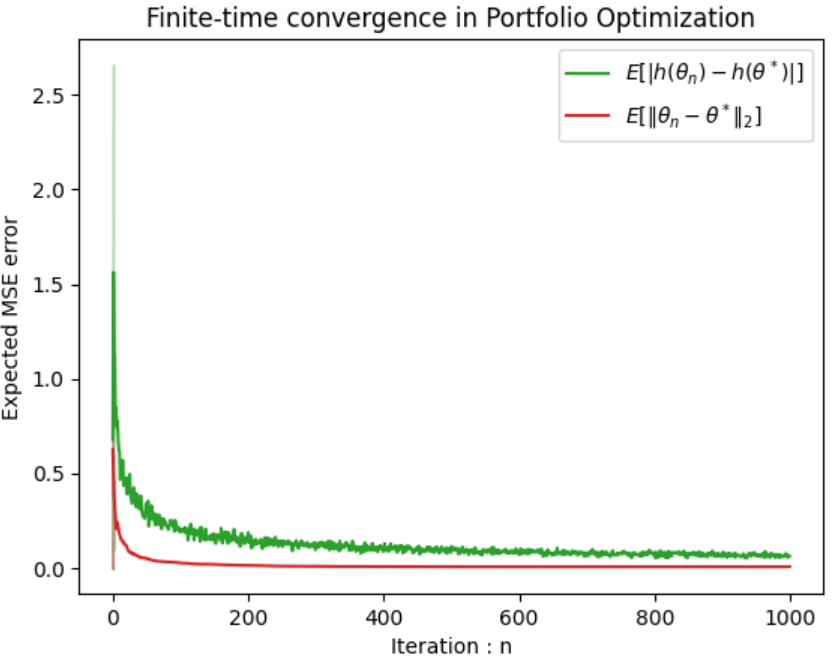}
}
    \caption{Convergence of UBSR-SGD in both, function value and parameter.}
    \label{fig:ubsr-sgd}
   \end{figure}
   
As $h(\cdot)$ is quadratic, positive-definiteness of $\Sigma$ implies that $h$ is strongly convex, which ensures that the bounds in \Cref{theorem:ubsr-gradient} are applicable. From \Cref{fig:ubsr-sgd}, we observe that the algorithm UBSR-SGD converges both in the function value as well as in the parameter value to the UBSR-optimal asset allocation.

\section{Conclusions}
\label{sec:conclusions}
We lay the foundations for UBSR estimation and optimization for the case of unbounded random variables.
 Our contributions are appealing in financial applications such as portfolio optimization as well as risk-sensitive reinforcement learning. 
 As future work, it would be interesting to (i) explore UBSR optimization in the non-convex case; and ii) develop Newton-based methods for UBSR optimization.



\bibliographystyle{IEEEtran}
\bibliography{IEEEabrv,bibliography}


\appendices
\label{sec:appendix}
\onecolumn
\section{Proofs for the claims in Section \ref{sec:pb}}
\label{app:prop1}
\subsection{Proof of Proposition \ref{prop:g}}
\label{proof:prop1}
\begin{proof}
    Recall that $g=\Exp[ l(-X-t)] - \lambda$. We first show that $g$ is continuous. Take any sequence $\{t_n\}_{n \geq 1} $ in $\Rel$ that converges to $t_0 \in \Rel$. We apply the dominated convergence theorem (DCT)\cite[Theorem 4.6.3]{durrett_2019} to obtain the following result.
\begin{align*}
    \lim_{n \to \infty} g(t_n)=\left[ \lim_{n \to \infty} \Exp \left[ l(-X-t_n)\right] \right] - \lambda 
    = \left[ \Exp \left[ \lim_{n \to \infty} l(-X-t_n) \right] \right] - \lambda 
    = \left[ \Exp \left[l(-X-t_0)\right] \right] -\lambda  
    = g(t_0).
\end{align*}
The interchange of expectation and the limit is due to the uniform integrability condition in $\xl$, while the equality following the interchange is due to $l$ being continuous. 

Now we show $g$ is strictly decreasing. Take any $t\in \Rel$ and consider the random variable $-l'(-X-t)$. From \Cref{as:2}, $l' \geq 0$ implies that the random variable $-l'(-X-t)$ is non-positive almost surely. \Cref{as:2} also implies that this random variable is strictly negative with positive probability, hence the expectation of this random variable is strictly negative. Thus in general, for every $t \in \Rel$, $\Exp[-l'(-X-t)] < 0$.
Now, taking derivative of $g$, we have
\begin{align}
    g'(t) =  \frac{\partial}{\partial t} \Exp\left[ l(-X-t) \right] 
          =  \Exp \left[\frac{\partial}{\partial t} l(-X-t) \right] 
          = \Exp \left[-l'(-X-t) \right] < 0,
\end{align}
where the second equality follows from DCT\cite[Theorem 4.6.3]{durrett_2019} via \Cref{as:3}. Note that $g$ is continuous and \Cref{as:1} holds. Then by Bolzano's theorem $g$ has a root in $[t_\mathrm{l},t_\mathrm{u}]$ and since $g$ is strictly decreasing, this root is unique, which we denote as $t^*$. Now we show that this root coincides with the UBSR of $X$.
\begin{align}
    \label{sr_definition}
    SR_{l,\lambda}(X) &\triangleq \inf \{ t \in R | \Exp[l(-X-t)] \leq \lambda \} \\
    &= \label{sr_g} \inf \{ t \in R | g
    (t) \leq 0 \} \\
    &= \label{sr_t} \inf \{ t \in R | t \geq t^* \} \\
    &= \nonumber t^*.
\end{align}
Assumption 1 ensures that the set in \cref{sr_definition} is non-empty and that the infimum is strictly greater than $-\infty$. On the other hand, $g$ being continuous and strictly decreasing ensures that the sets in (\ref{sr_g}) and (\ref{sr_t}) are identical.
\end{proof}
\subsection{Proof of Lemma \ref{lemma:sr-X-Y-bound}}
\label{proof:lemma_sr_X_Y_bound}
\begin{proof}
    Let $X,Y$ have finite $p$ moment. Recall that $\mathcal{H}(\mu_X,\mu_Y)$ denotes the set of all joint distributions whose marginals are $\mu_X$ and $\mu_Y$. By definition of $\mathcal{W}_p$ as the infimum, it follows that for every $\epsilon>0$, there exists $\eta' \in \mathcal{H}(\mu_X,\mu_Y)$ such that following holds.
    \begin{align*}
        \mathcal{W}_p(\mu_X,\mu_Y) &\geq \left(\Exp_{\eta'}\left[\norm{X-Y}_p^p\right]\right)^{1/p} - \epsilon \label{eq:was-p} \numberthis\\
        \mathcal{W}_p^p(\mu_X,\mu_Y) &\geq \Exp_{\eta'}\left[\norm{X-Y}_p^p\right] - \epsilon \label{eq:was-pp} \numberthis.
    \end{align*}
    Recall the definition $g_X(t) \triangleq \Exp\left[l(-X-t)\right] - \lambda$. Under \cref{as:l-smooth}, \cref{lemma:l-smooth-convex} implies that for every $t \in \Rel$, we have 
    \begin{align*}
        \left|g_X(t) - g_Y(t) \right| &\le \Exp_{\eta'}\left[\frac{L_2}{2}\norm{X-Y}_2^2 + a\norm{X}\norm{X-Y}_2^2\right] \\
        &\le \frac{L_2}{2}\Exp_{\eta'}\left[\norm{X-Y}_2^2\right] + a\sqrt{\Exp_{\eta'}\left[X^2\right]\Exp_{\eta'}\left[\norm{X-Y}_2^2\right]} \\
        &\le \frac{L_2}{2}\mathcal{W}_2^2(\mu_X,\mu_Y) + aT \mathcal{W}_2(\mu_X,\mu_Y).
    \end{align*}
    The second inequality follows directly from the Cauchy Schwartz's inequality, while the last inequality follows from \cref{eq:was-p,eq:was-pp}.
    Invoking \cref{lemma:sr-bounds} with $w = \frac{L_2}{2}\mathcal{W}_2^2(\mu_X,\mu_Y) + aT \mathcal{W}_2(\mu_X,\mu_Y)$, we have
    \begin{equation}
        \left|SR_{l,\lambda}(X) - SR_{l,\lambda}(Y) \right| \le \frac{L_2}{2b}\mathcal{W}_2^2(\mu_X,\mu_Y) + \frac{aT}{b} \mathcal{W}_2(\mu_X,\mu_Y).
    \end{equation}
\end{proof}
\subsection{Proof of Proposition \ref{proposition:convex_risk_measure}}
\label{proof:proposition_convexity_of_UBSR}
\begin{proof}Monotonicity: Let $X_1,X_2\in \xl$ such that $X_1 \leq X_2$ almost surely. Let $t_1 \triangleq SR_{l,\lambda}(X_1)$ and $t_2 \triangleq SR_{l,\lambda}(X_2)$. Then, we have $g_{X_1}(t_1) = g_{X_2}(t_2) = 0$. Since $l$ is increasing, we have $l(-X_1 - t_1) \geq l(-X_2 - t_1)$ almost surely. Taking expectation and subtracting $\lambda$, we have the following.
\begin{align}
    g_{X_1}(t_1) \geq g_{X_2}(t_1) 
    \implies g_{X_2}(t_2) \geq g_{X_2}(t_1) 
    \implies t_1 \geq t_2.
\end{align}
The first implication follows due to $g_{X_1}(t_1) = g_{X_2}(t_2)$, while the second implication follows because $g_{X_2}(\cdot)$ is strictly decreasing. 
Cash-invariance: Using the fact that for every $X \in \mathcal{X}_l$, $g_X(SR_{l,\lambda}(X)) = 0$, we have the following result:
\begin{align*}
    0 &=\Exp[l(-X - SR_{l,\lambda}(X))] - \lambda = \Exp[l(-(X+m) - (SR_{l,\lambda}(X)-m))] - \lambda  \\
    &= g_{X+m}(SR_{l,\lambda}(X)-m) \implies SR_{l,\lambda}(X+m) = SR_{l,\lambda}(X)-m,
\end{align*}
where the last equality follows from \Cref{prop:g}.

Convexity: We first claim that $\mathcal{A}_{l,\lambda}$ is convex, i.e., if $Y_1,Y_2 \in \mathcal{A}_{l,\lambda}$, then $SR_{l,\lambda}(\alpha Y_1 + (1-\alpha)Y_2) \le 0$, for every $\alpha \in [0,1]$. Let $Y_\alpha \triangleq \alpha Y_1+ (1-\alpha) Y_2$ and $t_\alpha \triangleq SR_{l,\lambda}(Y_\alpha)$. First, we use the fact that for every $\alpha \in [0,1], g_{Y_\alpha}(t_\alpha) = 0$, to show the following.
\begin{align}
    \nonumber 0 &= \Exp\left[ l(-Y_\alpha - t_\alpha)\right] - \lambda = \Exp\left[ l(\alpha(-Y_1 - t_\alpha) + (1-\alpha)(-Y_2 - t_\alpha)\right] - \lambda \\
    \label{eq:g-x1-x2}
    &\leq \alpha \Exp\left[l((-Y_1 - t_\alpha))\right] + (1-\alpha)\Exp \left[l(-Y_2 - t_\alpha)\right] - \lambda = \alpha g_{Y_1}(t_\alpha) + (1-\alpha) g_{Y_2}(t_\alpha),
\end{align}
where we use convexity of $l$ for the last inequality. Secondly, the monotonicity of $SR_{l,\lambda}(\cdot)$ implies that if $X \in \xl$ then $g_{X}(0) \le 0$. Since $Y_1,Y_2 \in \xl$ we have (a) $g_{Y_1}(0)\le 0$ and (b) $g_{Y_2}(0) \le 0$. \Cref{eq:g-x1-x2} implies that (c) $\alpha g_{Y_1}(t_\alpha) + (1-\alpha) g_{Y_2}(t_\alpha)\geq 0$ for any $\alpha \in [0,1]$. (a),(b),(c) can hold simultaneously only when \newline $t_\alpha = SR_{l,\lambda}(\alpha Y_1 + (1-\alpha)Y_2) \leq 0$ for every $\alpha \in [0,1]$. This implies $\mathcal{A}_{l,\lambda}$ is convex. 

Next we show that $SR_{l,\lambda}(\cdot)$ is convex. Cash invariance of  $SR_{l,\lambda}(\cdot)$ implies that if $X \in \xl$ then $X + SR_{l,\lambda}(X) \in \mathcal{A}_{l,\lambda}$. Let $X_1,X_2 \in \xl$. Then, substituting $Y_1 = X_1 + SR_{l,\lambda}(X_1)$ and $Y_2 = X_2 + SR_{l,\lambda}(X_2)$ in our claim that showed $\mathcal{A}_{l,\lambda}$ is convex, we have 
\begin{align*}
    &SR_{l,\lambda}\left( \alpha (X_1 + SR_{l,\lambda}(X_1)) + (1-\alpha )(X_2 + SR_{l,\lambda}(X_2))\right) \le 0 \\
    \implies &SR_{l,\lambda}\big(\alpha (X_1)+(1-\alpha)(X_2) + \alpha SR_{l,\lambda}(X_1) + (1-\alpha)SR_{l,\lambda}(X_2)\big) \leq 0 \\
    \implies &SR_{l,\lambda}\big(\alpha (X_1)+(1-\alpha)(X_2)\big) \le \alpha SR_{l,\lambda}(X_1) + (1-\alpha)SR_{l,\lambda}(X_2)  \implies SR_{l,\lambda}(\cdot) \text{ is convex}.
\end{align*}
\end{proof}
\section{Proofs for the claims in Section \ref{sec:ubsr-est}}
\subsection{Proof of Lemma \ref{lemma:sr-x-y-exp-1}}
\label{proof:lemma-sr-x-y-exp-1}
\begin{proof}
    Let $\mu$ denote the distribution of $X$, and let $\mu_m$ be its $m$-sample empirical distribution. Let $\hat{X}$ denote a random variable that takes one of the values $\{Z_i\}_{i=1}^m$ with equal probability. Thus $\hat{X}$ follows the distribution $\mu_m$. Then, $SR_m(Z) = \inf\left\{t \in R \left| \frac{1}{m}\sum_{i=1}^ml(-Z_i-t)-\lambda\right.\right\} =\inf\left\{ t \in R \left|\Exp\left[l(-\hat{X}-t)\right] - \lambda \right.\right\} = SR_{l,\lambda}(\hat{X})$. Then from \cref{eq:sr-bound-lipschitz}, we have 
    \begin{equation}\label{eq:sr-X-Z-1}
        \left| SR_{l,\lambda}(X) - SR_m(Z) \right| \le \left| SR_{l,\lambda}(X) - SR_{l,\lambda}(\hat{X}) \right|\le \frac{L_1}{b}\mathcal{W}_1\left(\mu, \mu_m\right).
    \end{equation}
    Taking expectation on both sides, we have
    \begin{align}
        \Exp\left[\left| SR_{l,\lambda}(X) - SR_m(Z) \right|\right] \leq \frac{L_1}{b}\Exp\left[\mathcal{W}_1(\mu,\mu_m)\right] \leq \frac{48TL_1}{b\sqrt{m}},
    \end{align}
    where the final inequality follows from \cite[Theorem 2.1]{fournier-2023} with $p=1,d=1,m=1,q=5$ to show that following holds: $\Exp\left[\mathcal{W}_1(\mu,\mu_m)\right] = \Exp\left[\mathcal{T}_1(\mu,\mu_m)\right] \leq \frac{48T}{\sqrt{m}}$.

    In a similar manner, squaring \cref{eq:sr-X-Z-1}, and taking expectation on both sides, we have
    \begin{align}
        \Exp\left[\left| SR_{l,\lambda}(X) - SR_m(Z) \right|^2\right]  \le \frac{L_1^2}{b^2}\Exp\left[\left(\mathcal{W}_1(\mu,\mu_m)\right)^2\right] \le \frac{L_1^2}{b^2}\Exp\left[\mathcal{W}^2_2(\mu,\mu_m)\right] \leq \frac{108T^2L_1^2}{b^2\sqrt{m}}. 
    \end{align}
    Here, the first inequality follows from the fact that $(y_1+y_2)^2\le 2y_1^2+2y_2^2$. The second inequality follows directly from \cref{lemma:lp-inequality} with the choice of $p=1$. The final inequality follows from \cite[Theorem 2.1]{fournier-2023} with $p=2,d=1,m=1,q=5$ to show that following holds: $\Exp\left[\mathcal{W}_2^2(\mu,\mu_m)\right] = \Exp\left[\mathcal{T}_2(\mu,\mu_m)\right] \leq \frac{108T^2}{\sqrt{m}}$.
\end{proof}
\subsection{Proof of Lemma \ref{lemma:sr-x-y-exp-2}}
\label{proof:lemma-sr-x-y-exp-2}
\begin{proof}
    Using similar argument as in the proof of \cref{lemma:sr-x-y-exp-1}, we invoke \cref{lemma:sr-X-Y-bound} to have the following bound.
    \begin{equation}\label{eq:sr-X-Z-2}
        \left| SR_{l,\lambda}(X) - SR_m(Z) \right| \le \frac{L_2}{2b}\mathcal{W}_2^2(\mu_X,\mu_m) + \frac{aT}{b} \mathcal{W}_2(\mu_X,\mu_m)
    \end{equation}
    Taking expectation on both sides, we have
    \begin{align}
        \Exp&\left[\left| SR_{l,\lambda}(X) - SR_m(Z) \right|\right] \le \frac{L_2}{2b}\Exp\left[\mathcal{W}_2^2(\mu_X,\mu_m)\right] + \frac{aT}{b} \Exp\left[\mathcal{W}_2(\mu_X,\mu_m)\right] \\
        &\le \frac{L_2}{2b}\Exp\left[\mathcal{W}_2^2(\mu_X,\mu_m)\right] + \frac{a T}{b} \sqrt{\Exp\left[\mathcal{W}^2_2(\mu_X,\mu_m)\right]} \le \frac{27L_2T^2}{b\sqrt{m}} + \frac{a T^2 \sqrt{54}}{bm^{1/4}} \le \frac{(27L_2 + 8 a) T^2}{bm^{1/4}}. 
    \end{align}
    The second inequality follows from from Jenson's inequality while the third inequality follows from \cite[Theorem 2.1]{fournier-2023} with $p=2,d=1,m=1,q=5$ to show that following holds: $\Exp\left[\mathcal{W}_2^2(\mu,\mu_m)\right] = \Exp\left[\mathcal{T}_2(\mu,\mu_m)\right] \leq \frac{54T^2}{\sqrt{m}}$.
    In a similar manner, squaring \cref{eq:sr-X-Z-2} and taking expectations on both sides, we have
    \begin{align}
        \Exp&\left[\left| SR_{l,\lambda}(X) - SR_m(Z) \right|^2\right] \le \frac{L_2^2}{b^2}\Exp\left[\left(\mathcal{W}_2^2(\mu_X,\mu_m)\right)^2\right] + \frac{2 a^2 T^2}{b^2} \Exp\left[\mathcal{W}_2^2(\mu_X,\mu_m)\right] \\
        &\le \frac{L_2^2}{b^2}\Exp\left[\mathcal{W}_4^4(\mu_X,\mu_m)\right] + \frac{2 a^2 T^2}{b^2} \Exp\left[\mathcal{W}_2^2(\mu_X,\mu_m)\right] \le \frac{540L_2^2T^4}{b^2\sqrt{m}} + \frac{108 a^2 T^4}{b^2 \sqrt{m}} \le \frac{(540L_2^2+108 a^2) T^4}{b^2 \sqrt{m}}. 
    \end{align}
    Here, the first inequality follows from the fact that $(a+b)^2\le 2a^2+2b^2$. The second inequality follows directly from \cref{lemma:lp-inequality} with the choice of $p=2$. The final inequality follows from \cite[Theorem 2.1]{fournier-2023} with $p=4,d=1,m=1,q=9$ and then with $p=2,d=1,m=1,q=9$ to show that following holds: 
    $\Exp\left[\mathcal{W}_4^4(\mu,\mu_m)\right] = \Exp\left[\mathcal{T}_4(\mu,\mu_m)\right] \leq \frac{540T^4}{\sqrt{m}}$ and $\Exp\left[\mathcal{W}_2^2(\mu,\mu_m)\right] = \Exp\left[\mathcal{T}_2(\mu,\mu_m)\right] \leq \frac{54T^2}{\sqrt{m}}$.
\end{proof}

\subsection{Proof of Proposition \ref{proposition:sr-estimator}}\label{proof:proposition-sr-estimator}
\begin{proof}
    Using the triangle inequality, we have
    \begin{align*}
        &\Exp[|\hat{t}_m - SR_{\lambda,l}(X)|]   
        =\Exp[|\hat{t}_m - SR^m_{\lambda,l}(Z) + SR^m_{\lambda,l}(Z) - SR_{\lambda,l}(X)|] \\
        &\leq \Exp[|\hat{t}_m - SR^m_{\lambda,l}(Z)| + |SR^m_{\lambda,l}Z) - SR_{\lambda,l}(X)|]
        \leq \delta + \Exp[|SR^m_{\lambda,l}(Z) - SR_{\lambda,l}(X)|] 
        \leq \frac{d_1+C_1}{\sqrt{m}}.
    \end{align*}
    where the last inequality follows from \Cref{lemma:sr-x-y-exp-1}.
    Similarly, using the identity: $(a+b)^2\le 2a^2+2b^2$, we show the following result.
    \begin{align*}
    &\Exp[\left(\hat{t}_m - SR_{\lambda,l}(X)\right)^2]  
        =\Exp[\left(\hat{t}_m - SR^m_{\lambda,l}(Z) + SR^m_{\lambda,l}(Z) - SR_{\lambda,l}(X)\right)^2] \\
        &\leq 2\Exp[\left(\hat{t}_m - SR^m_{\lambda,l}(Z)\right)^2 + \left(SR^m_{\lambda,l}(Z) - SR_{\lambda,l}(X)\right)^2]
        \leq 2\delta^2 + 2\Exp[\left(SR^m_{\lambda,l}(Z) - SR_{\lambda,l}(X)\right)^2] 
        \leq \frac{2 (d_1^2 + C_2)}{m}.
    \end{align*}
\end{proof}
\section{Proofs for the claims in Section \ref{sec:ubsr-opt}}\label{proofs:ubsr-opt}

\subsection{Proof of Lemma \ref{lemma:h-convex}}\label{proof:lemma_h_cvx}
\begin{proof}Since $F$ is $\mu$-strongly concave w.p. 1, we have the following for every $\theta_1,\theta_2 \in \Theta$ and every $\alpha \in [0,1]$. 
\begin{equation}\label{eq:f-sc}
    F(\alpha \theta_1 +(1-\alpha)\theta_2,\xi) \geq \alpha F(\theta_1,\xi) + (1-\alpha) F(\theta_2,\xi) + \frac{\alpha(1-\alpha)\mu}{2}\norm{\theta_1-\theta_2}^2\qquad \text{w.p. 1}.
\end{equation} 
Then we have,
\begin{align*}\label{mono}
    h(\alpha \theta_1 + (1-\alpha) \theta_2) &= SR_{\lambda,l}\left(F\left(\alpha \theta_1 + (1-\alpha)\theta_2, \xi\right)\right) \\
    &\leq SR_{\lambda,l}\left(\alpha F\left(\theta_1,\xi\right) + (1-\alpha) F\left(\theta_2,\xi\right) + \frac{\alpha(1-\alpha)\mu}{2}\norm{\theta_1-\theta_2}^2\right) \\
    &= SR_{\lambda,l}\left(\alpha F\left(\theta_1,\xi\right) + (1-\alpha) F\left(\theta_2,\xi\right)\right) - \frac{\alpha(1-\alpha)\mu}{2}\norm{\theta_1-\theta_2}^2 \\
    &\leq \alpha SR_{\lambda,l}\left(F\left(\theta_1,\xi\right)\right) + (1-\alpha) SR_{\lambda,l}\left(F\left(\theta_2,\xi\right)\right) - \frac{\alpha(1-\alpha)\mu}{2}\norm{\theta_1-\theta_2}^2 \\
    &= \alpha h(\theta_1) + (1-\alpha) h(\theta_2) - \frac{\alpha(1-\alpha)\mu}{2}\norm{\theta_1-\theta_2}^2.
\end{align*}
Here, the first inequality follows from monotonicity in \cref{eq:f-sc}, while the second equality and the second inequality follow from cash-invariance and convexity respectively. Thus, $h$ is $\mu$-strongly convex. 
\end{proof}

\subsection{Proof of Lemma \ref{lemma:dct}}\label{proof:lemma_dct}
\begin{proof}
    Continuity of $g$ from \Cref{prop:g}, and \Cref{as:3_with_theta} satisfy the necessary conditions for invoking the dominated convergence theorem (see \cite[Theorem A.5.3]{durrett_2019}). This theorem guarantees that $g$ is continuously differentiable. Under \Crefrange{as:3_with_theta}{as:f_c_diff}, the partial derivative expressions in \crefrange{part_der}{part_der2} follow by applying the chain rule. 
\end{proof}
\subsection{Proof of Lemma \ref{lemma:ift}}\label{proof:lemma_ift}
\begin{proof}
Note that a) from \Cref{lemma:dct}, $g$ is continuously differentiable in $\Rel \times \Theta$, b) from \Cref{A_increasing}, $\frac{\partial g}{\partial t} < 0$ and c) from \Cref{P_main}, $\forall \theta \in \Theta,g(h(\theta), \theta) = 0$.
   Using (a), (b) and (c), we invoke the implicit function theorem (cf  \cite[Theorem 9.28]{real-analysis-rudin}) to infer the main claim in the \cref{eq:implicit} of \cref{lemma:ift}.
\end{proof}
\subsection{Proof of Theorem \ref{theorem:ubsr-gradient}}\label{proof:theorem-1}
\begin{proof}
From \Cref{lemma:ift}, for any $\tilde{\theta} \in \Theta$, we have
\begin{equation}
    \frac{\partial h(\theta)}{\partial \theta_i}\Big|_{\tilde{\theta}} = -\frac{\partial g(t,\theta) / \partial \theta_i}{\partial g(t,\theta) / \partial t}\Big|_{h(\tilde{\theta}), \tilde{\theta}} \quad \forall i \in \{1,2,\ldots,d\}.
\end{equation} 
From \Cref{lemma:dct}, we conclude that $\forall i \in \{1,2,\ldots,d\},$
\begin{equation}\label{h_partial}
\frac{\partial h(\theta)}{\partial \theta_i} = - \frac{\Exp\Big[ l'(-F(\theta,\xi)-h(\theta)) \frac{\partial F(\theta,\xi)}{\partial \theta_i} \Big]}{\Exp\big[ l'(-F(\theta,\xi)-h(\theta)) \big]}.
\end{equation}
\end{proof}

\
\subsection{Proof of Lemma \ref{lemma:grad_est_bound}}\label{proof:grad_est_bound} 
\begin{proof}
Let $A,A',A_m$ and $B,B',B_m$ be defined as follows:
    \begin{align*}
        &A \triangleq \Exp \Big[ l'(-F(\theta,\xi)-h(\theta)) \nabla F(\theta,\xi) \Big], \quad B \triangleq \Exp \big[ l'(-F(\theta,\xi)-h(\theta)) \big],\\
        &A' \triangleq \frac{1}{m}\sum_{j=1}^m \Big[ l'(-F(\theta,Z_j)-h(\theta)) \nabla F(\theta,Z_j)\Big], \quad B' \triangleq \frac{1}{m} \sum_{j=1}^m \big[ l'(-F(\theta,Z_j)-h(\theta)) \big],\\
        &A_m \triangleq \frac{1}{m}\sum_{j=1}^m \Big[ l'(-F(\theta,Z_j)-SR^m_{\theta}(\mathbf{\hat{Z}})) \nabla F(\theta,Z_j)\Big], \quad B_m \triangleq \frac{1}{m} \sum_{j=1}^m \big[ l'(-F(\theta,Z_j)-SR^m_{\theta}(\mathbf{\hat{Z}})) \big].\\ 
    \end{align*}
    Let  $1 \leq q \leq p \leq \infty$ hold.

    \textit{Step 1:} We derive general bounds on the $q^{th}$ moments of random vectors $A_m-A'$ and $B_m-B'$, having finite $\norm{\cdot}_{\Leb^p}$ norms. We shall use this general result to obtain specific bounds for different values of $p$ and $q$. With $\norm{\cdot}_p^q$ denoting $q^{th}$ power of the $p$-norm of a vector, we have
    \begin{align}
        \nonumber \Exp &\left[ \norm{A_m - A'}_p^q \right] \\
        \nonumber &= \Exp \left[ \norm{\frac{1}{m}\sum_{j=1}^m \left[ \big[ l'(-F(\theta,Z_j)-SR^m_{\theta}(\mathbf{\hat{Z}})) - l'(-F(\theta,Z_j)-h(\theta)) \big] \nabla F(\theta,Z_j)\right]}_p^q\right] \\
        \nonumber &= \frac{1}{m^q} \Exp \left[ \norm{ \sum_{j=1}^m \left[ \big[ l'(-F(\theta,Z_j)-SR^m_{\theta}(\mathbf{\hat{Z}})) - l'(-F(\theta,Z_j)-h(\theta)) \big] \nabla F(\theta,Z_j)\right]}_p^q\right] \\
        \label{eq:mq-lemma} &\leq \frac{m^{q-1}}{m^q}\Exp \left[ \sum_{j=1}^m \norm{ \big[ l'(-F(\theta,Z_j)-SR^m_{\theta}(\mathbf{\hat{Z}})) - l'(-F(\theta,Z_j)-h(\theta)) \big] \nabla F(\theta,Z_j)}_p^q\right] \\
        \nonumber &= \frac{1}{m}\Exp \left[ \sum_{j=1}^m  \big|l'(-F(\theta,Z_j)-SR^m_{\theta}(\mathbf{\hat{Z}})) - l'(-F(\theta,Z_j)-h(\theta)) \big|^q \norm{\nabla F(\theta,Z_j)}_p^q\right] \\
        \nonumber &\leq \frac{L_2^q}{m}\Exp \left[\sum_{j=1}^m  \big|SR^m_{\theta}(\mathbf{\hat{Z}}) - h(\theta) \big|^q \norm{\nabla F(\theta,Z_j)}_p^q\right] \\
        \label{eq:z-independence} &= \frac{L_2^q}{m} \left[\sum_{j=1}^m  \Exp \left[ \big|SR^m_{\theta}(\mathbf{\hat{Z}}) -h(\theta) \big|^q\right] \Exp \left[\norm{\nabla F(\theta,Z_j)}_p^q \right] \right] \\
        &\leq \label{I1} M^qL_2^q \Exp \big[ \big| SR^m_{\theta}(\mathbf{\hat{Z}})-h(\theta) \big|^q \big],
    \end{align}
    where the proof inequality in \cref{eq:mq-lemma} follows from \Cref{lemma-normed-sum} (see Appendix \ref{ss:supporting-lemmas}). The equality in \cref{eq:z-independence} uses the independence of $\Zh$ and $\Z$, while the final inequality can be trivially obtained from \Cref{as:lp_moments_F_dF}. 
    
    Using similar arguments, it can be shown that
    \begin{align} \label{I3}
        \Exp &\left[ \left| B_m - B' \right|^q \right] \leq L_2^q \Exp \big[ \big| SR^m_{\theta}(\mathbf{\hat{Z}})-h(\theta) \big|^q \big].
    \end{align}
    \textit{Step 2:} With an additional condition of $q \leq 2$, we can further obtain following bounds on $A'-A$ and $B'-B$ respectively. Let $H_\theta(z) \triangleq l'(-F(\theta,z)-h(\theta))\nabla F(\theta,z)$.
    \begin{align}
        \nonumber &\Exp \left[ \norm{A' - A}_p^q \right] 
        = \Exp \left[ \norm{\frac{1}{m}\sum_{j=1}^m \left[ l'(-F(\theta,Z_j)-h(\theta)) \nabla F(\theta,Z_j) \right] - \Exp\left[l'(-F(\theta,\xi)-h(\theta)) \nabla F(\theta,\xi)\right]}_p^q\right] \\
         &= \Exp \left[ \norm{\frac{1}{m}\sum_{j=1}^m H_\theta(Z_j) - \Exp\left[H_\theta(\xi)\right]}_p^q\right] 
        \leq \left[\Exp \left[ \norm{\frac{1}{m}\sum_{j=1}^m H_\theta(Z_j) - \Exp\left[H_\theta(\xi)\right]}_p^2\right]\right]^\frac{q}{2} 
        \leq \label{I2} \left[ \frac{d\sigma_1^2}{m}\right]^\frac{q}{2} = \frac{(\sqrt{d}\sigma_1)^q}{m^{q/2}}.
    \end{align}
    Here the first inequality is a consequence of Lyapunov's inequality while the second inequality is due to the variance bound in \Cref{sigma_bound}. 
    
    Using similar arguments, it can be shown that
    \begin{equation}\label{I4}
        \Exp \left[ \left| B' - B \right|_p^q \right] \leq \frac{\sigma_2^q}{m^{q/2}}.
    \end{equation}
    \textit{Step 3:} Combining the results from \textit{steps 1 and 2}, we have
    \begin{align}
        \nonumber \norm{ {J_\theta(\mathbf{Z},\mathbf{\hat{Z}}) - \nabla h(\theta)}}_p &= \norm{ B_m^{-1}A_m - B^{-1}A}_p 
        = \norm{(BB_m)^{-1}(BA_m - AB_m)}_p \\ 
        \nonumber &\leq \frac{\norm{BA_m - AB_m}_p}{b^2} = \frac{\norm{B(A_m - A) - A(B - B_m)}_p}{b^2} \\
        \nonumber &\leq \frac{|B|\norm{A_m - A}_p + |B - B_m| \norm{A}_p}{b^2} \\
        \nonumber &\leq \frac{L_1\norm{A_m - A}_p + L_1 M|B - B_m|}{b^2} \\
        &\leq \label{eq_pq} \frac{L_1}{b^2} \left[ \norm{A_m - A'}_p + \norm{A' - A}_p \right] + \frac{L_1M}{b^2}\left[|B - B'| + |B' - B_m|\right],
    \end{align}
    where the first inequality follows from \Cref{as:l-Lipschitz}, the second inequality follows from Minkowski's inequality applied in conjunction with the Cauchy-Schwartz inequality, while the third inequality follows from \Cref{as:l-Lipschitz,as:lp_moments_F_dF}.
    Applying $p=1$ above and taking expectation on both sides, we have
    \begin{align*}
        \Exp \left[ \norm{ {J_\theta(\mathbf{Z},\mathbf{\hat{Z}}) - \nabla h(\theta)}}_1 \right] &\leq \frac{L_1}{b^2} \left[ \Exp \norm{A_m - A'}_1 +  \Exp \left[\norm{A' - A}_1 \right] \right] \\
        &+ \frac{L_1 M}{b^2} \left[ \Exp\left[|B_m - B'| \right] + \Exp\left[|B' - B| \right]  \right].
    \end{align*}
    Using $p=q=1$ in \cref{I1,I2,I3,I4}, the above inequality can be expressed as 
    \begin{align*}
        \Exp \left[ \norm{ {J_\theta(\mathbf{Z},\mathbf{\hat{Z}}) - \nabla h(\theta)}}_1 \right] &\leq \frac{2L_1ML_2}{b^2}\Exp\left[ \big| SR^m_{\theta}(\mathbf{\hat{Z}})-h(\theta) \big|\right] + \frac{L_1(\sigma_1\sqrt{d}+M\sigma_2)}{b^2\sqrt{m}}.
    \end{align*}
    Applying \Cref{lemma_srh}, we obtain
    \begin{equation}
        \Exp \left[ \norm{ {J_\theta(\mathbf{Z},\mathbf{\hat{Z}}) - \nabla h(\theta)}}_1 \right] \leq  \frac{2C_1L_1ML_2}{b^2\sqrt{m}} + \frac{L_1(\sigma_1\sqrt{d}+M\sigma_2)}{b^2\sqrt{m}} = \frac{D_1}{\sqrt{m}}.
    \end{equation}
    With $p=2$ in \cref{eq_pq} and using the fact that $\norm{a+b}^2 \leq 2\norm{a}^2 + 2\norm{b}^2$, we obtain the following bound. 
    \begin{align*}
        \nonumber \Exp &\left[ \norm{ {J_\theta(\mathbf{Z},\mathbf{\hat{Z}}) - \nabla h(\theta)}}_2^2 \right] \\
        &\leq \frac{4L_1^2}{b^4} \left[ \Exp \norm{A_m - A'}_2^2 +  \Exp \left[\norm{A' - A}_2^2 \right] \right] 
        + \frac{4L_1^2M^2}{b^4} \left[ \Exp\left[|B_m - B'|^2 \right] + \Exp\left[|B' - B|^2 \right]  \right].
    \end{align*}
    Using $p=q=2$ in \cref{I1,I2,I3,I4}, we have the following.
    \begin{align}
        \nonumber \Exp \left[ \norm{ {J_\theta(\mathbf{Z},\mathbf{\hat{Z}}) - \nabla h(\theta)}}_2^2 \right] &\leq \frac{8L_1^2M^2L_2^2}{b^4}\Exp\left[ \big| SR^m_{\theta}(\mathbf{\hat{Z}})-h(\theta) \big|^2\right] + \frac{4L_1^2(d\sigma_1^2+M^2\sigma_2^2)}{b^4 m} \\
        &\leq \frac{8L_1^2M^2L_2^2C_2}{b^4m}+ \frac{4L_1^2(d\sigma_1^2+M^2\sigma_2^2)}{b^4 m} 
        = \frac{D_2}{m}. 
    \end{align} 
\end{proof}

\subsection{Proof of Theorem \ref{theorem:sgd_convergence}}\label{proof:theorem-2}
\begin{proof}
Let $z_n, y_{n-1}, M_n$ be defined as follows:
\begin{align}
    z_n &\triangleq \theta_n - \theta^*, \\
    y_{n-1} &\triangleq J_{\theta_{n-1}}(\mathbf{Z^{n}},\mathbf{\hat{Z}^{n}}) - \nabla h(\theta_{n-1}), \text{ and } \\
    M_n &\triangleq \int_0^1 \nabla^2 h(m\theta_n + (1-m)\theta^*) dm. 
\end{align}
With repeated invocations of the fundamental theorem of calculus, we get $\nabla h(\theta_{n-1}) = M_{n-1} z_{n-1}$.
\paragraph{One-step recursion}W.p. $1$, we have
\begin{align}
    \nonumber \norm{z_n}_2 &= \norm{\Pi_\Theta \left( \theta_{n-1} - \alpha_{n} J_{\theta_{n-1}}(\mathbf{Z^{n}},\mathbf{\hat{Z}^{n}})\right) - \theta^*}_2 \\
    \label{wo_proj} &\leq \norm{\theta_{n-1} - \alpha_{n} J_{\theta_{n-1}}(\mathbf{Z^{n}},\mathbf{\hat{Z}^{n}}) - \theta^*}_2 \\
    \nonumber &= \norm{z_{n-1} - \alpha_{n}\left( y_{n-1} + \nabla h(\theta_{n-1}) \right)}_2 \\
    \nonumber &= \norm{z_{n-1} - \alpha_{n} y_{n-1} - \alpha_n M_{n-1} z_{n-1}}_2 \\ 
    \nonumber &= \norm{\left( I - \alpha_{n} M_{n-1} \right)z_{n-1} - \alpha_{n} y_{n-1}}_2 \\
    \nonumber &\leq \norm{\left( I - \alpha_{n} M_{n-1} \right)z_{n-1}}_2 + \alpha_{n}\norm{y_{n-1}}_2 \\
    \nonumber &\leq \norm{I - \alpha_{n} M_{n-1}}\norm{z_{n-1}}_2 + \alpha_{n}\norm{y_{n-1}}_2 \\
    \nonumber &\leq \left| 1 - \alpha_{n}\mu \right| \norm{z_{n-1}}_2 + \alpha_{n}\norm{y_{n-1}}_2.
\end{align}
The inequality in \cref{wo_proj} is the result of $\Pi_\Theta$ being non-expansive, see \cite{DBLP:journals/corr/abs-2101-02137}[Lemma 10].
\paragraph{After complete unrolling}W.p. $1$, we have the 
\begin{align}
    \norm{z_n}_2 &\leq \left[ \prod_{i=1}^n \left| 1 - \alpha_{i}\mu \right| \right]\norm{z_{0}}_2 + \sum_{k=1}^{n}\alpha_{k} \left[ \prod_{i=k+1}^n \left| 1 - \alpha_{i}\mu \right| \right] \norm{y_{k-1}}_2 \\
    \label{eq:Q-intro}&= Q_{1:n}\norm{z_{0}}_2 + \sum_{k=1}^{n}\alpha_{k} Q_{k+1:n} \norm{y_{k-1}}_2,
\end{align}
where we used strong convexity of $h$ and the condition $\mu c \geq 1$ to get the following.
\paragraph*{case 1} For $k\geq 2$, with the condition $\mu c \leq 3$ the product in $Q_{k+1:n}$ contains only positive terms. Then we have,
\begin{align*}
    Q_{k+1:n}\leq \prod_{i = k+1}^{n} (1-\alpha_i \mu) \leq \exp\left( - \sum_{i=k+1}^n \alpha_i \mu \right) = \exp\left( -\mu c \sum_{i=k+1}^n \frac{1}{i} \right) \leq \left(\frac{k+1}{n+1}\right)^{\mu c} \leq \frac{k+1}{n+1},
\end{align*}
where the last inequality uses the condition that $\mu c \geq 1$.
\paragraph*{case 2} For $k=1$, using the result from previous case, we have
\begin{align*}
    Q_{k+1:n}\leq |(1-\alpha_2 \mu)| \prod_{i = 3}^{n} (1-\alpha_i \mu) \leq (1+\frac{\mu c}{2}) \frac{3}{n+1} \leq \frac{4(k+1)}{n+1}.
\end{align*}
\paragraph*{case 3}For $k=0$, we apply similar technique to get the following bound. 
\begin{align*}
    Q_{k+1:n}\leq |(1-\alpha_1 \mu)| \cdot |(1-\alpha_2 \mu)| \prod_{i = 3}^{n} (1-\alpha_i \mu) \leq 2(1+\frac{\mu c}{2}) \frac{3}{n+1} \leq \frac{15}{n+1}.
\end{align*}
Squaring both sides in \cref{eq:Q-intro}, applying the fact that $(a+b)^2 \leq 2a^2+2b^2$, and using the bounds from the three cases, we have the following w.p. $1$.
\begin{align*}
    \norm{z_n}_2^2 &\leq \frac{450\norm{z_{0}}_2^2}{(n+1)^2} + 2 \left( \sum_{k=1}^{n}4 \alpha_{k} \left(\frac{k+1}{n+1}\right) \norm{y_{k-1}}_2^2\right)^2\\
    &\leq \frac{450\norm{z_{0}}_2^2}{(n+1)^2} + \frac{64c^2}{(n+1)^2} \underbrace{\left( \sum_{k=1}^{n} \norm{y_{k-1}}_2^2\right)^2}_{I_1}.
\end{align*}
Taking expectation on both sides, we have the following.
\begin{align}\label{I0}
    \Exp\left[\norm{z_n}_2^2 \right] \leq \frac{450\Exp\left[\norm{\theta_0 - \theta^*}_2^2\right]}{(n+1)^2} + \frac{64c^2}{(n+1)^2}\Exp\left[ I_1 \right].
\end{align}
We bound $I_1$ as follows:
\begin{align*}
    I_1 &= \left(\sum_{k=1}^{n} \norm{y_{k-1}}_2 \right)^2 
    = \left(\sum_{k=1}^{n}\norm{y_{k-1}}_2^2 + \sum_{k=1}^{n} \underset{l \neq k}{\sum_{l=1,}^{n}} \norm{y_{k-1}}_2\norm{y_{l-1}}_2\right)  \\
    &=\left( \sum_{k=1}^{n}\norm{y_{k-1}}_2^2 + 2\sum_{k=1}^{n} \sum_{l=k+1}^{n} \norm{y_{k-1}}_2 \norm{y_{l-1}}_2\right).
\end{align*}
Note that for every $n \in \mathbf{N}$, $\theta_{n-1}$ is $\mathcal{F}_{n-1}$-measurable and that $\sigma(\mathbf{Z^{n}},\mathbf{\hat{Z}^{n}})$ and $\mathcal{F}_{n-1}$ are independent. Using this fact, we take expectation on both sides to get the following result.
\begin{align*}
   \Exp \left[ I_1 \right] &\leq \left[ \sum_{k=1}^{n}\Exp
   \left[\norm{y_{k-1}}_2^2\right] + 2\sum_{k=1}^{n} \sum_{l=k+1}^{n} \Exp \left[ \norm{y_{k-1}}_2\norm{y_{l-1}}_2 \right]\right] \\
   &= \left[ \sum_{k=1}^{n}\Exp \left[\Exp
   \left[ \left. \norm{y_{k-1}}_2^2 \right| \mathcal{F}_{k-1} \right]\right] + 2\sum_{k=1}^{n} \sum_{l=k+1}^{n} \Exp\left[\Exp \left[ \left. \norm{y_{k-1}}_2\norm{y_{l-1}}_2\right| \mathcal{F}_{l-1} \right]\right]\right] \\
   &= \left[ \sum_{k=1}^{n}\Exp \left[\Exp
   \left[ \left. \norm{y_{k-1}}_2^2 \right| \mathcal{F}_{k-1} \right]\right] + 2\sum_{k=1}^{n} \sum_{l=k+1}^{n} \Exp\left[ \norm{y_{k-1}}_2 \Exp \left[ \left. \norm{y_{l-1}}_2\right| \mathcal{F}_{l-1} \right]\right]\right]\\
   &\leq \left[ \sum_{k=1}^{n}\frac{\hat{D}_2}{m_k} + 2\sum_{k=1}^{n} \sum_{l=k+1}^n \Exp\left[ \norm{y_{k-1}}_2 \cdot \frac{\hat{D}_1}{\sqrt{m_l}}\right]\right]\\
   \label{m_split} &\leq \left[\hat{D}_2 \sum_{k=1}^{n}\frac{1}{m_k} + 2\hat{D}_1\sum_{k=1}^{n}  \Exp \left[ \norm{y_{k-1}}_2\right]\sum_{l=k+1}^n \frac{1}{\sqrt{m_l}}\right]\\
   &\leq \left[\hat{D}_2 \sum_{k=1}^{n}\frac{1}{k} + 2\hat{D}_1\sum_{k=1}^{n}  \Exp\left[\Exp \left[ \left. \norm{y_{k-1}}_2\right] \right| \mathcal{F}_{k-1}\right] \sum_{l=k+1}^n \frac{1}{\sqrt{l}}\right]\\
   &\leq \left[\hat{D}_2 \left(1+ \int_{k=1}^{n}\frac{dk}{k}\right) + 2\hat{D}_1\sum_{k=1}^{n}  \frac{\hat{D}_1}{\sqrt{m_k}} \int_{l=0}^n \frac{dl}{\sqrt{l}}\right]\\
   &\leq \left[ \hat{D}_2 \left(1+\log(n)\right) + 4 \hat{D}_1^2\sqrt{n} \int_{k=0}^{n}  \frac{dk}{\sqrt{k}} \right]\\
   &\leq \left[ \hat{D}_2 \left(1+\log(n)\right) + 8 \hat{D}_1^2n \right].
\end{align*}
Combining this inequality with \cref{I0}, we have the following result.
\begin{equation}
    \Exp\left[\norm{z_n}_2^2 \right] \leq \frac{450\Exp\left[\norm{\theta_0 - \theta^*}_2^2\right]+64c^2\hat{D}_2\left(1+\log(n)\right)}{(n+1)^2} + \frac{512c^2\hat{D}_1^2}{n+1}.
\end{equation}
Using constant batch-size of $m_k=m$, we have the following bound.
\begin{equation}
    \Exp\left[\norm{z_n}_2^2 \right] \leq \frac{450\Exp\left[\norm{\theta_0 - \theta^*}_2^2\right]+64c^2\hat{D}_2}{(n+1)m} + \frac{64c^2\hat{D}_1^2}{m}.
\end{equation}
\end{proof}

\subsection{Proof of Corollary \ref{corr:h-smooth}}
\label{proof:corr_h_smooth}
\begin{proof}
With the gradient expression of $h$ cerived in \Cref{theorem:ubsr-gradient}, we have
\begin{align*}
    \nonumber \norm{ \nabla h(\theta_1) - \nabla  h(\theta_2) }_2 
    &= \norm{ - \frac{\Exp\Big[ l'(-F(\theta_1,\xi)-h(\theta_1)) \nabla_\theta F(\theta_1,\xi) \Big]}{\Exp\big[ l'(-F(\theta_1,\xi)-h(\theta_1)) \big]} +  \frac{\Exp\Big[ l'(-F(\theta_2,\xi)-h(\theta_2)) \nabla_\theta F(\theta_2,\xi) \Big]}{\Exp\big[ l'(-F(\theta_2,\xi)-h(\theta_2)) \big]} }_2 \\
    &\leq \norm{ \frac{\Exp\Big[ l'(-F(\theta_2,\xi)-h(\theta_2)) \nabla_\theta F(\theta_2,\xi) \Big]}{\Exp\big[ l'(-F(\theta_2,\xi)-h(\theta_2)) \big]} - \frac{\Exp\Big[ l'(-F(\theta_1,\xi)-h(\theta_1)) \nabla_\theta F(\theta_2,\xi) \Big]}{\Exp\big[ l'(-F(\theta_1,\xi)-h(\theta_1)) \big]} }_2 \\
    &+ \norm{ \frac{\Exp\Big[ l'(-F(\theta_1,\xi)-h(\theta_1)) \nabla_\theta F(\theta_2,\xi) \Big]}{\Exp\big[ l'(-F(\theta_1,\xi)-h(\theta_1)) \big]} - \frac{\Exp\Big[ l'(-F(\theta_1,\xi)-h(\theta_1)) \nabla_\theta F(\theta_1,\xi) \Big]}{\Exp\big[ l'(-F(\theta_1,\xi)-h(\theta_1)) \big]} }_2 \\
    &\leq \norm{ \frac{\Exp\big[ \big( l'(-F(\theta_2,\xi)-h(\theta_2)) - l'(-F(\theta_1,\xi)-h(\theta_1)) \big) \big( \nabla_\theta F(\theta_2,\xi) \big) \big]}{ \Exp\big[ l'(-F(\theta_2,\xi)-h(\theta_2)) \big]} }_2 \\
    &+ \norm{ \frac{\Exp\big[ \big( l'(-F(\theta_1,\xi)-h(\theta_1)) \nabla_\theta F(\theta_2,\xi) \big) \big]\Exp\big[ \big( l'(-F(\theta_2,\xi)-h(\theta_2)) - l'(-F(\theta_1,\xi)-h(\theta_1)) \big) \big]}{ \Exp\big[ l'(-F(\theta_2,\xi)-h(\theta_2)) \big] \Exp\big[ l'(-F(\theta_1,\xi)-h(\theta_1)) \big]} }_2 \\
    &+ \norm{ \frac{\Exp\Big[ l'(-F(\theta_1,\xi)-h(\theta_1)) \big( \nabla_\theta F(\theta_2,\xi) - \nabla_\theta F(\theta_2,\xi) \big) \Big]}{\Exp\big[ l'(-F(\theta_1,\xi)-h(\theta_1)) \big]} }_2 \\
    \nonumber &\leq \frac{L_2 M}{b} + \frac{L_1 M L_2(L_3+L_4)}{b^2} \norm{\theta_2-\theta_1}_2 \\
    &+ \frac{L_1 \big|\big| \nabla_\theta F(\theta_2,\xi) - \nabla_\theta F(\theta_2,\xi) \big| \big|_{L_2}}{b}\\
    &\leq \Big( \frac{L_2 M\left(\frac{1+L_1(L_3+L_4)}{b}\right) + L_1 L_5}{b} \Big) \norm{\theta_2 - \theta_1}_2.
\end{align*}
    Here the second inequality is a standard algebraic identity, while the third follows from \Cref{lemma:uvw,lemma:l-prime-Lipschitz}.
 \end{proof}
 
\section{Supporting lemmas}\label{ss:supporting-lemmas}
\subsection{Inequality about \texorpdfstring{L\textsubscript{p} moments}.}
\begin{lemma}\label{lemma:lp-inequality}
    Let the random variables $X$ and $Y$ have finite $2p$ moment for some $p \geq 1$. Then,
    \begin{equation*}
        \mathcal{T}_{2p}(\mu_{\mathbf{X}},\mu_{\mathbf{Y}}) \geq \left(\mathcal{T}_p(\mu_{\mathbf{X}},\mu_{\mathbf{Y}})\right)^2, 
    \end{equation*}
    where $\mu_X,\mu_Y$ are marginal distributions of $X$ and $Y$ respectively.
\end{lemma}
\begin{proof} Recall that $\mathcal{T}_p(\mu_{\mathbf{X}},\mu_{\mathbf{Y}}) \triangleq \inf \left\{ \int \norm{x-y}^p \eta(dx,dy) : \eta \in \mathcal{H}(\mu_{\mathbf{X}},\mu_{\mathbf{Y}}) \right\}$. Let $\eta_1, \eta_2 \in \mathcal{H}(\mu_X,\mu_Y)$ be the infimal joint distributions for $\mathcal{T}_p$ and $\mathcal{T}_{2p}$ respectively. Then we have,
\begin{align*}
    \mathcal{T}_{2p} = \Exp_{\eta_2}\left[\norm{X-Y}^{2p}\right] \geq \left(\Exp_{\eta_2}\left[\norm{X-Y}^{p}\right]\right)^2 \geq \left(\Exp_{\eta_1}\left[\norm{X-Y}^{p}\right]\right)^2 = \left(\mathcal{T}_{p}\right)^2. 
\end{align*}
\end{proof}
\subsection{Classical quadratic bound on smooth and convex functions}
\begin{lemma}\label{lemma:l-smooth-convex}
    Suppose $l:\Rel^d \to \Rel$ is a convex and $L$-smooth function such that that exists $c>0$ that satisfies $\norm{\nabla l(x)} \le c \norm{x}$ for every $x \in \Rel^d$. Then for every $x,y \in \Rel^d$, we have
    \begin{equation}\label{eq:l-diff-bounds-conv-smth}
        \left| l(y) - l(x) \right| \le \frac{L}{2}\norm{y-x}^2 + c \norm{x}\norm{y-x}.
    \end{equation}
\end{lemma}
\begin{proof}
    Convexity and smoothness of $l$ implies the following two inequalities respectively.
    \begin{equation}
        {\nabla l(x)}^T\left(y - x\right) \leq l(y) - l(x) \le \frac{L}{2}\norm{y-x}^2 + {\nabla l(x)}^T \left(y - x\right).
    \end{equation}
    Under the sub-linearity of the gradient from the assumptions of the lemma, \cref{eq:l-diff-bounds-conv-smth} follows directly from the above equation.
\end{proof}
\begin{lemma}\label{lemma:sr-bounds}
    Suppose \Cref{as:b-lower-bound} holds for some $b>0$ and there exists $w < \infty$ such that the random variables $X,Y \in \xl$ satisfy $\left|g_X(t) - g_Y(t) \right| \leq w, \forall t \in \Rel$. Then following holds.
    \begin{equation}\label{eq:sr-bounds}
        \left| SR_{l,\lambda}(X) - SR_{l,\lambda}(Y) \right| \leq \frac{w}{b}.
    \end{equation}
\end{lemma}
\begin{proof}
    Suppose $t_1 \geq t_2$, then \Cref{as:b-lower-bound} implies the following.
    \begin{equation}\label{eq:gy-bound-chain}
        t_1 \geq t_2 \implies -Y-t_2 \geq -Y-t_1 \implies l(-Y-t_2) - l(-Y-t_1) \geq b(t_1-t_2) 
        \implies g_Y(t_2) - g_Y(t_1) \geq b(t_1-t_2).
    \end{equation}
    \Cref{as:b-lower-bound} holds for any $Y\in \xl$. Similarly, let $X \in \xl$. Take any $\epsilon_1>0$. We now have the following.
    \begin{align*}
    &SR_{l,\lambda}(X) + \epsilon_1 > SR_{l,\lambda}(X) > SR_{l,\lambda}(X) - \epsilon_1 \\
        \implies &g_X\left(SR_{l,\lambda}(X) + \epsilon_1\right) < g_X\left(SR_{l,\lambda}(X)\right) < g_X\left(SR_{l,\lambda}(X) - \epsilon_1\right) \\
        \implies &g_X\left(SR_{l,\lambda}(X) + \epsilon_1\right) < g_Y\left(SR_{l,\lambda}(Y)\right) < g_X\left(SR_{l,\lambda}(X) - \epsilon_1\right) \\
        \implies &g_Y\left(SR_{l,\lambda}(X) + \epsilon_1\right) - w < g_Y\left(SR_{l,\lambda}(Y)\right) < g_Y\left(SR_{l,\lambda}(X) - \epsilon_1\right) + w. \label{eq:sr-bounds-w} \numberthis
    \end{align*}
    Without loss of generality, let $SR_{l,\lambda}(X)>SR_{l,\lambda}(Y)$. Let $c = SR_{l,\lambda}(X) - SR_{l,\lambda}(Y)$. 
    Take any $\epsilon_2 \in \left(0,c\right)$. Substituting $t_1 = SR_{l,\lambda}(X)-\epsilon_2$ and $t_2=SR_{l,\lambda}(Y)$ in \cref{eq:gy-bound-chain}, and $\epsilon_1 = \epsilon_2$ in the second inequality of \cref{eq:sr-bounds-w}, we have the following. 
    \begin{align}
        b\left(SR_{l,\lambda}(X)-SR_{l,\lambda}(Y)-\epsilon_2\right) < g_Y\left(SR_{l,\lambda}\left(Y\right)\right) - g_Y\left(SR_{l,\lambda}(X)-\epsilon_2)\right) < w         \implies SR_{l,\lambda}(X)-SR_{l,\lambda}(Y) < \frac{w}{b} + \epsilon_2.
    \end{align}
    Since we can choose any $\epsilon_2 \in \left(0,c\right)$, we have $\left|SR_{l,\lambda}(X)-SR_{l,\lambda}(Y)\right| \leq \frac{w}{b}$. 
\end{proof}

\subsection{Claims in the proof of Lemma \ref{lemma:grad_est_bound}}
\begin{lemma}\label{lemma:l-prime-Lipschitz}
    Suppose \Cref{as:1_with_theta,A_increasing,as:3_with_theta,as:f_c_diff,as:lp_moments_F_dF} hold. Then 
    \begin{equation}
        \Exp\left[|l'\left(-F(\theta_1,\xi)-h(\theta_1)\right)-l'\left(-F(\theta_2,\xi)-h(\theta_2)\right)|\right] \leq L_2(L_2+L_4) \norm{\theta_2-\theta_1}_2
    \end{equation}
\end{lemma}
\begin{proof}
    Since the assumptions of \Cref{theorem:ubsr-gradient} hold, we have
    \begin{align*}
        &\Exp\left[|l'\left(-F(\theta_1,\xi)-h(\theta_1)\right)-l'\left(-F(\theta_2,\xi)-h(\theta_2)\right)|\right] 
        \leq L_2 \Exp\left[|F(\theta_2,\xi) - F(\theta_1,\xi) + h(\theta_2)-h(\theta_1)|\right] \\
        &\leq L_2 \Exp\left[|F(\theta_2,\xi) - F(\theta_1,\xi)|\right] + L_4 \norm{\theta_2-\theta_1}_2 
        \leq L_2(L_2+L_4) \norm{\theta_2-\theta_1}_2.
    \end{align*}
    The inequalities above follow from the Lipschitz continuity of the functions $l',F,h$. 
\end{proof}
\begin{lemma}\label{lemma-normed-sum}
    Given vectors $\{a_i\}_{i=1}^m$ in any normed space $\norm{\cdot}$, then for any $n\ge 1$ following holds:
    \begin{equation}
         \norm{\sum_m a_i}^n \leq  m^{n-1} \sum_m \norm{a_i}^n.
    \end{equation}
\end{lemma}

\begin{proof}
    \begin{align*}
        \norm{\sum_m a_i}^n \leq \left(\sum_m \norm{a_i}\right)^n 
        = m^n \left(\sum_m \frac{\norm{a_i}}{m}\right)^n 
        \leq m^n \frac{\sum_m \norm{a_i}^n}{m} 
        = m^{n-1} \sum_m \norm{a_i}^n
    \end{align*}
    Here the first inequality is the Minkowski's inequality while the second inequality follows from the Jenson's inequality applied to the function $x \mapsto x^n$, which is convex for $x \geq 0, n \ge 1$.
\end{proof}

\subsection{Lemma on bounding the norm of Hessian of a smooth function}
\begin{lemma}\label{lemma_h_hessian}
    Let $r:\Rel^d \to \Rel$ be twice differentiable function, whose gradient is $L_1$-Lipschitz and $L_2$-Lipschitz w.r.t the vector 1-norm and vector 2-norm respectively. Then for every $x \in \Rel^d$, the following holds:
    \begin{equation}
        \norm{\nabla^2 r(x)} \leq L_1, \textrm{ and }
        \norm{\nabla^2 r(x)}_F \leq L_2\sqrt{d}.
    \end{equation}
\end{lemma}
\begin{proof}
    Let $I_k$ indicate the $k^{th}$ row of the identity matrix $I$. The Hessian matrix of $r$ is given as follows: $\forall i \in [1,2,\ldots,d], \forall j \in [1,2,\ldots,d]$, 
\begin{align*}
    \nabla^2_{i,j} r(x) \triangleq  \frac{\partial^2 r(x)}{\partial x_i \partial x_j} &= \lim_{h \to 0} {\frac{\frac{\partial r(x+hI_i)}{\partial x_j} - \frac{\partial r(x)}{\partial x_j}}{h}}. 
\end{align*}
Then $\forall i \in [1,2,\ldots,d]$, the $i^{th}$ row of the Hessian, $\nabla^2_i r(x)$, is
\begin{equation}\label{hess_comp}
    \nabla^2_i r(x) = \lim_{h \to 0} \frac{\nabla r(x+hI_i) - \nabla r(x)}{h}.
\end{equation}

Taking the vector 2-norm, and under the Lipschitz assumption on $r$, we have $\forall i \in [1,2,\ldots,d]$, 
\begin{align*}
    \norm{\nabla^2_i r(x)}_2 = \norm{\lim_{h \to 0} \frac{\nabla r(x+hI_i) - \nabla r(x)}{h}} 
    = \lim_{h \to 0} \norm{\frac{\nabla r(x+hI_i) - \nabla r(x)}{h}} 
    \leq \lim_{h \to 0} \frac{L_2\norm{hI_i}_2}{h} = L_2. 
\end{align*}
Thus we have
\begin{align}
    \norm{\nabla^2 r(x)}_\mathrm{F} = \left(\sum_{i=1}^d \norm{\nabla^2_i r(x)}_2^2\right)^{1/2} \leq \sqrt{dL_2^2} = L_2\sqrt{d}.
\end{align}

Similarly, applying vector 1-norm to eq (\ref{hess_comp}), and using Lipschitz property we have for all rows $i \in [1,2,\ldots,d],\norm{\nabla^2_i r(x)}_1 \leq L_1$. Using this and the fact that Hessian is symmetric (since $\nabla r$ is Lipschitz continuous), we apply standard matrix identities to get the following bounds.
\begin{align*}
    \norm{\nabla^2 r(x)}_1 &= \norm{\nabla^2 r(x)}_\infty \leq \max_i \norm{\nabla^2_i r(x)}_1 \leq L_1. \\
    \norm{\nabla^2 r(x)}^2 &\leq \norm{\nabla^2 r(x)}_1 \norm{\nabla^2 r(x)}_\infty \leq L_1^2.
\end{align*}
Hence proved.
\end{proof}

\subsection{Claims in the proof of Corollary \ref{corr:h-smooth}}
\begin{lemma}\label{lemma:uvw}
    Let $p \in [1,\infty]$. Suppose $U$ is a real-valued random variable whose absolute value is bounded by $w>0$ and $\mathbf{V}$ be an $n$-dimensional random vector whose $p^{th}$ moment exists and is finite. Then, 
    \begin{equation}
        \norm{\Exp[U\mathbf{V}]}_p \leq w \norm{\mathbf{V}}_{L_p}
    \end{equation}
\end{lemma}
\begin{proof}We use Cauchy-Schwarz inequality and the fact that the $p^{th}$ moment of the random vector is finite, to obtain the following result.
    \begin{align*}
    \norm{\Exp[U\mathbf{V}]}_p = \Bigg( \sum_{i=1}^n \Big| \Exp\left[UV_i\right]  \Big|^p \Bigg)^\frac{1}{p} 
    \leq \Bigg( \sum_{i=1}^n \Exp\left[ | UV_i |\right]^p  \Bigg)^\frac{1}{p} 
    \leq w \Bigg( \sum_{i=1}^n \Exp\left[ | V_i|\right]^p \Bigg)^\frac{1}{p} 
    = w \Bigg( \Exp\left[ \sum_{i=1}^n | V_i|^p \right]\Bigg)^\frac{1}{p}
    = w \norm{\mathbf{V}}_{L_p}.
\end{align*}
\end{proof}

\addtolength{\textheight}{-12cm}   


\end{document}